\documentclass{article}
\usepackage[T1]{fontenc}
\usepackage{microtype}
\usepackage{graphicx}
\usepackage{booktabs}
\usepackage{array}
\usepackage{longtable}
\usepackage{pdflscape}
\usepackage{hyperref}

\usepackage[preprint]{icml2026}
\hypersetup{pdfsubject={Preprint}}
\usepackage{amsmath}
\usepackage{amssymb}
\usepackage{mathtools}
\usepackage{amsthm}
\icmltitlerunning{Beyond Outcome Rewards: Post-Rationale Distribution Optimization}

\makeatletter
\newcommand{\npj}[1]{\ifcsname npj@#1\endcsname\csname npj@#1\endcsname\else\PackageError{npj}{undefined number #1}{}\textbf{??#1??}\fi}
\expandafter\def\csname npj@C.phi-4-mini.mp.aurc.diff\endcsname{$-0.106$}
\expandafter\def\csname npj@C.phi-4-mini.mp.brier.diff\endcsname{$-0.068$}
\expandafter\def\csname npj@C.qwen3-0.6b.gv.aurc.diff\endcsname{$-0.093$}
\expandafter\def\csname npj@C.qwen3-0.6b.gv.aurc.p\endcsname{$P\leq0.002$}
\expandafter\def\csname npj@C.qwen3-0.6b.mp.aurc.diff\endcsname{$-0.161$}
\expandafter\def\csname npj@C.qwen3-1.7b.mp.brier.diff\endcsname{$-0.080$}
\expandafter\def\csname npj@C.sec.phi-4-mini.mp.acc.pp\endcsname{$+4.1$}
\expandafter\def\csname npj@C.sec.phi-4-mini.mp.eaurc.diff\endcsname{$-0.080$}
\expandafter\def\csname npj@C.sec.qwen3-0.6b.gv.acc.pp\endcsname{$-2.7$}
\expandafter\def\csname npj@C.sec.qwen3-0.6b.gv.acc.ppabs\endcsname{$2.7$}
\expandafter\def\csname npj@C.sec.qwen3-1.7b.mp.acc.pp\endcsname{$+2.6$}
\expandafter\def\csname npj@C.sec.qwen3-1.7b.mp.eaurc.diff\endcsname{$-0.132$}
\expandafter\def\csname npj@PH.X1.mean.cn_s1_lam0.single.raw.jsd\endcsname{$0.0673$}
\expandafter\def\csname npj@PH.X1.mean.cn_s2_cont.single.raw.jsd\endcsname{$0.0675$}
\expandafter\def\csname npj@PH.X1.mean.cn_s2_correct.single.raw.jsd\endcsname{$0.0628$}
\expandafter\def\csname npj@PH.X1.mean.cn_s2_rl.single.raw.jsd\endcsname{$0.0649$}
\expandafter\def\csname npj@PH.X1.proper_vs_continuation_raw.single.jsd.diff\endcsname{$-0.003$}
\expandafter\def\csname npj@PH.X1.proper_vs_continuation_raw.single.jsd.hier\endcsname{$-0.007$ to $0.002$}
\expandafter\def\csname npj@PH.X1.proper_vs_continuation_raw.single.jsd.p\endcsname{$P<0.001$}
\expandafter\def\csname npj@PH.X1.proper_vs_continuation_raw.single.jsd.seeds\endcsname{3 of 3}
\expandafter\def\csname npj@PH.X1.proper_vs_continuation_raw.single.sqdist.p\endcsname{$P=0.15$}
\expandafter\def\csname npj@PH.X1.proper_vs_sampled_raw.single.jsd.ci\endcsname{$0.0003$ to $0.004$}
\expandafter\def\csname npj@PH.X1.proper_vs_sampled_raw.single.jsd.diff\endcsname{$+0.002$}
\expandafter\def\csname npj@PH.X1.proper_vs_sampled_raw.single.jsd.hier\endcsname{$-0.003$ to $0.007$}
\expandafter\def\csname npj@PH.X1.proper_vs_sampled_raw.single.jsd.p\endcsname{$P=0.022$}
\expandafter\def\csname npj@PH.X1.proper_vs_sampled_raw.single.jsd.seeds\endcsname{2 of 3}
\expandafter\def\csname npj@PH.X1.proper_vs_sampled_ts.mixture.jsd.p\endcsname{$P=0.04$}
\expandafter\def\csname npj@PH.X1.proper_vs_sampled_ts.mixture.jsd.seeds\endcsname{1 of 3}
\expandafter\def\csname npj@PH.X1.proper_vs_sampled_ts.single.sqdist.p\endcsname{$P=0.069$}
\expandafter\def\csname npj@PH.X2.grpocont_vs_grpo_06b.qwen3-0.6b.gv.acc.p\endcsname{$P=0.14$}
\expandafter\def\csname npj@PH.X2.grpocont_vs_grpo_06b.qwen3-0.6b.gv.acc.pp\endcsname{$+1.1$}
\expandafter\def\csname npj@PH.X2.grpocont_vs_grpo_06b.qwen3-0.6b.gv.aurc.diff\endcsname{$-0.110$}
\expandafter\def\csname npj@PH.X2.grpocont_vs_grpo_06b.qwen3-0.6b.gv.aurc.p\endcsname{$P<0.001$}
\expandafter\def\csname npj@PH.X2.grpocont_vs_grpo_06b.qwen3-0.6b.gv.brier.diff\endcsname{$-0.041$}
\expandafter\def\csname npj@PH.X2.grpocont_vs_grpo_06b.qwen3-0.6b.mp.acc.p\endcsname{$P<0.001$}
\expandafter\def\csname npj@PH.X2.grpocont_vs_grpo_06b.qwen3-0.6b.mp.acc.pp\endcsname{$+5.0$}
\expandafter\def\csname npj@PH.X2.grpocont_vs_grpo_06b.qwen3-0.6b.mp.aurc.diff\endcsname{$-0.130$}
\expandafter\def\csname npj@PH.X2.grpocont_vs_grpo_06b.qwen3-0.6b.mp.brier.diff\endcsname{$-0.056$}
\expandafter\def\csname npj@PH.X2.grpocont_vs_s1cont_06b.qwen3-0.6b.gv.acc.cipp\endcsname{$2.7$ to $6.0$}
\expandafter\def\csname npj@PH.X2.grpocont_vs_s1cont_06b.qwen3-0.6b.gv.acc.pp\endcsname{$+4.4$}
\expandafter\def\csname npj@PH.X2.grpocont_vs_s1cont_06b.qwen3-0.6b.gv.aurc.diff\endcsname{$-0.020$}
\expandafter\def\csname npj@PH.X2.grpocont_vs_s1cont_06b.qwen3-0.6b.gv.aurc.p\endcsname{$P=0.0014$}
\expandafter\def\csname npj@PH.X2.grpocont_vs_s1cont_06b.qwen3-0.6b.gv.brier.diff\endcsname{$-0.037$}
\expandafter\def\csname npj@PH.X2.grpocont_vs_s1cont_06b.qwen3-0.6b.mp.acc.p\endcsname{$P=0.78$}
\expandafter\def\csname npj@PH.X2.grpocont_vs_s1cont_06b.qwen3-0.6b.mp.acc.pp\endcsname{$-0.2$}
\expandafter\def\csname npj@PH.X2.grpocont_vs_s1cont_06b.qwen3-0.6b.mp.aurc.diff\endcsname{$+0.038$}
\expandafter\def\csname npj@PH.X2.grpocont_vs_s1cont_06b.qwen3-0.6b.mp.aurc.hier\endcsname{$0.009$ to $0.069$}
\expandafter\def\csname npj@PH.X2.grpocont_vs_s1cont_06b.qwen3-0.6b.mp.aurc.p\endcsname{$P<0.001$}
\expandafter\def\csname npj@PH.X2.grpocont_vs_s1cont_06b.qwen3-0.6b.mp.brier.diff\endcsname{$+0.015$}
\expandafter\def\csname npj@PH.X2.grpocont_vs_s1cont_06b.qwen3-0.6b.mp.brier.hier\endcsname{$0.00003$ to $0.030$}
\expandafter\def\csname npj@PH.X2.grpocont_vs_s1cont_06b.qwen3-0.6b.mp.brier.p\endcsname{$P=0.0014$}
\expandafter\def\csname npj@PH.X2.s1cont_vs_stage1.qwen3-0.6b.gv.acc.pp\endcsname{$-0.6$}
\expandafter\def\csname npj@PH.X2.s1cont_vs_stage1.qwen3-0.6b.mp.acc.p\endcsname{$P=0.46$}
\expandafter\def\csname npj@PH.X2.twostage_vs_s1cont.nfavour\endcsname{8}
\expandafter\def\csname npj@PH.X2.twostage_vs_s1cont.ntests\endcsname{12}
\expandafter\def\csname npj@PH.X2.twostage_vs_s1cont.qwen3-0.6b.gv.acc.cipp\endcsname{$1.0$ to $3.5$}
\expandafter\def\csname npj@PH.X2.twostage_vs_s1cont.qwen3-0.6b.gv.acc.hierpp\endcsname{$-0.7$ to $5.1$}
\expandafter\def\csname npj@PH.X2.twostage_vs_s1cont.qwen3-0.6b.gv.acc.p\endcsname{$P=0.0078$}
\expandafter\def\csname npj@PH.X2.twostage_vs_s1cont.qwen3-0.6b.gv.acc.pp\endcsname{$+2.2$}
\expandafter\def\csname npj@PH.X2.twostage_vs_s1cont.qwen3-0.6b.gv.acc.seeds\endcsname{3 of 3}
\expandafter\def\csname npj@PH.X2.twostage_vs_s1cont.qwen3-0.6b.gv.aurc.diff\endcsname{$-0.012$}
\expandafter\def\csname npj@PH.X2.twostage_vs_s1cont.qwen3-0.6b.gv.aurc.p\endcsname{$P=0.13$}
\expandafter\def\csname npj@PH.X2.twostage_vs_s1cont.qwen3-0.6b.gv.brier.diff\endcsname{$-0.019$}
\expandafter\def\csname npj@PH.X2.twostage_vs_s1cont.qwen3-0.6b.gv.brier.hier\endcsname{$-0.045$ to $0.008$}
\expandafter\def\csname npj@PH.X2.twostage_vs_s1cont.qwen3-0.6b.gv.brier.p\endcsname{$P=0.032$}
\expandafter\def\csname npj@PH.X2.twostage_vs_s1cont.qwen3-0.6b.gv.brier.seeds\endcsname{2 of 3}
\expandafter\def\csname npj@PH.X2.twostage_vs_s1cont.qwen3-1.7b.gv.acc.p\endcsname{$P=1.0$}
\expandafter\def\csname npj@PH.X2.twostage_vs_s1cont.qwen3-1.7b.gv.acc.pp\endcsname{$-0.1$}
\expandafter\def\csname npj@PH.X2.twostage_vs_s1cont.qwen3-1.7b.mp.brier.p\endcsname{$P=0.28$}
\expandafter\def\csname npj@PH.X3.offline_noanchor_vs_anchored.qwen3-0.6b.gv.acc.pp\endcsname{$+5.1$}
\expandafter\def\csname npj@PH.X3.offline_noanchor_vs_anchored.qwen3-0.6b.gv.aurc.diff\endcsname{$-0.021$}
\expandafter\def\csname npj@PH.X3.offline_noanchor_vs_anchored.qwen3-0.6b.gv.aurc.p\endcsname{$P=0.005$}
\expandafter\def\csname npj@PH.X3.offline_noanchor_vs_anchored.qwen3-0.6b.gv.brier.diff\endcsname{$-0.044$}
\expandafter\def\csname npj@PH.X3.offline_noanchor_vs_anchored.qwen3-1.7b.gv.aurc.diff\endcsname{$-0.022$}
\expandafter\def\csname npj@PH.X3.offline_noanchor_vs_anchored.qwen3-1.7b.gv.aurc.p\endcsname{$P=0.005$}
\expandafter\def\csname npj@PH.X3.offline_noanchor_vs_anchored.qwen3-1.7b.gv.aurc.seeds\endcsname{2 of 3}
\expandafter\def\csname npj@PH.X3.own_noanchor_vs_anchored.qwen3-0.6b.gv.acc.p\endcsname{$P<0.001$}
\expandafter\def\csname npj@PH.X3.own_noanchor_vs_anchored.qwen3-0.6b.gv.acc.pp\endcsname{$+3.8$}
\expandafter\def\csname npj@PH.X3.own_noanchor_vs_anchored.qwen3-0.6b.gv.aurc.diff\endcsname{$-0.022$}
\expandafter\def\csname npj@PH.X3.own_noanchor_vs_anchored.qwen3-0.6b.gv.brier.diff\endcsname{$-0.042$}
\expandafter\def\csname npj@PH.X3.own_noanchor_vs_anchored.qwen3-0.6b.mp.aurc.diff\endcsname{$+0.022$}
\expandafter\def\csname npj@PH.X3.own_noanchor_vs_anchored.qwen3-0.6b.mp.aurc.p\endcsname{$P=0.004$}
\expandafter\def\csname npj@PH.X3.own_noanchor_vs_anchored.qwen3-0.6b.mp.brier.diff\endcsname{$+0.013$}
\expandafter\def\csname npj@PH.X3.own_noanchor_vs_anchored.qwen3-0.6b.mp.brier.p\endcsname{$P=0.001$}
\expandafter\def\csname npj@PH.X3.own_noanchor_vs_anchored.qwen3-1.7b.gv.aurc.diff\endcsname{$-0.020$}
\expandafter\def\csname npj@PH.X3.own_noanchor_vs_anchored.qwen3-1.7b.gv.aurc.p\endcsname{$P=0.074$}
\expandafter\def\csname npj@PH.X3.own_vs_offline_noanchor.qwen3-1.7b.gv.aurc.diff\endcsname{$+0.005$}
\expandafter\def\csname npj@PH.X3.own_vs_offline_noanchor.qwen3-1.7b.gv.aurc.p\endcsname{$P=0.9$}
\expandafter\def\csname npj@PH.X3.own_vs_offline_noanchor.qwen3-1.7b.gv.brier.diff\endcsname{$-0.009$}
\expandafter\def\csname npj@PH.X3.own_vs_offline_noanchor.qwen3-1.7b.mp.aurc.ci\endcsname{$-0.028$ to $-0.005$}
\expandafter\def\csname npj@PH.X3.own_vs_offline_noanchor.qwen3-1.7b.mp.aurc.diff\endcsname{$-0.016$}
\expandafter\def\csname npj@PH.X3.own_vs_offline_noanchor.qwen3-1.7b.mp.aurc.hier\endcsname{$-0.034$ to $0.001$}
\expandafter\def\csname npj@PH.X3.own_vs_offline_noanchor.qwen3-1.7b.mp.aurc.p\endcsname{$P=0.04$}
\expandafter\def\csname npj@PH.X3.own_vs_offline_noanchor.qwen3-1.7b.mp.aurc.seeds\endcsname{3 of 3}
\expandafter\def\csname npj@PH.X3.own_vs_offline_noanchor.qwen3-1.7b.mp.brier.diff\endcsname{$-0.010$}
\expandafter\def\csname npj@PH.X3.own_vs_offline_noanchor.qwen3-1.7b.mp.brier.p\endcsname{$P=0.1$}
\expandafter\def\csname npj@PH.X3.rationale.qwen3-0.6b.gv.noanchor.forced\endcsname{$54.7$}
\expandafter\def\csname npj@PH.X3.rationale.qwen3-0.6b.gv.noanchor_refrat.forced\endcsname{$47.7$}
\expandafter\def\csname npj@PH.X3.rationale.qwen3-0.6b.gv.warm400.forced\endcsname{$3.2$}
\expandafter\def\csname npj@PH.X3.rationale.qwen3-0.6b.mp.noanchor.s17.forced\endcsname{$16.0$}
\expandafter\def\csname npj@PH.X3.rationale.qwen3-0.6b.mp.noanchor.s17.tokens\endcsname{$144$}
\expandafter\def\csname npj@PH.X3.rationale.qwen3-0.6b.mp.noanchor.s23.forced\endcsname{$97.4$}
\expandafter\def\csname npj@PH.X3.rationale.qwen3-0.6b.mp.noanchor.s23.tokens\endcsname{$13$}
\expandafter\def\csname npj@PH.X3.rationale.qwen3-0.6b.mp.noanchor.s29.forced\endcsname{$7.5$}
\expandafter\def\csname npj@PH.X3.rationale.qwen3-0.6b.mp.noanchor.s29.tokens\endcsname{$112$}
\expandafter\def\csname npj@PH.X3.rationale.qwen3-0.6b.mp.noanchor_refrat.forced\endcsname{$13.1$}
\expandafter\def\csname npj@PH.X3.rationale.qwen3-0.6b.mp.noanchor_refrat.length_limit\endcsname{$7.3$}
\expandafter\def\csname npj@PH.X3.rationale.qwen3-0.6b.mp.refrat.forced\endcsname{$15.5$}
\expandafter\def\csname npj@PH.X3.rationale.qwen3-0.6b.mp.refrat.length_limit\endcsname{$15.3$}
\expandafter\def\csname npj@PH.X3.rationale.qwen3-0.6b.mp.warm400.s17.forced\endcsname{$16.6$}
\expandafter\def\csname npj@PH.X3.rationale.qwen3-0.6b.mp.warm400.s17.tokens\endcsname{$121$}
\expandafter\def\csname npj@PH.X3.rationale.qwen3-0.6b.mp.warm400.s29.forced\endcsname{$17.7$}
\expandafter\def\csname npj@PH.X3.rationale.qwen3-0.6b.mp.warm400.s29.tokens\endcsname{$126$}
\expandafter\def\csname npj@PH.X3.rationale.qwen3-1.7b.gv.noanchor.forced\endcsname{$44.6$}
\expandafter\def\csname npj@PH.X3.rationale.qwen3-1.7b.gv.noanchor_refrat.forced\endcsname{$55.6$}
\expandafter\def\csname npj@PH.X3.rationale.qwen3-1.7b.gv.warm400.forced\endcsname{$8.8$}
\expandafter\def\csname npj@PH.X3.rationale.qwen3-1.7b.mp.noanchor.forced\endcsname{$48.7$}
\expandafter\def\csname npj@PH.X3.rationale.qwen3-1.7b.mp.noanchor.length_limit\endcsname{$46.5$}
\expandafter\def\csname npj@PH.X3.rationale.qwen3-1.7b.mp.noanchor_refrat.forced\endcsname{$73.9$}
\expandafter\def\csname npj@PH.X3.rationale.qwen3-1.7b.mp.noanchor_refrat.length_limit\endcsname{$73.0$}
\expandafter\def\csname npj@PH.X3.rationale.qwen3-1.7b.mp.warm400.forced\endcsname{$26.8$}
\expandafter\def\csname npj@PH.X3.trainforced.qwen3-0.6b.gv.noanchor\endcsname{$51.9$}
\expandafter\def\csname npj@PH.X3.trainforced.qwen3-0.6b.gv.noanchor_refrat\endcsname{$2.8$}
\expandafter\def\csname npj@PH.X3.trainforced.qwen3-0.6b.gv.warm400\endcsname{$2.4$}
\expandafter\def\csname npj@PH.X3.trainforced.qwen3-1.7b.gv.noanchor\endcsname{$38.5$}
\expandafter\def\csname npj@PH.X3.trainforced.qwen3-1.7b.gv.noanchor_refrat\endcsname{$8.8$}
\expandafter\def\csname npj@PH.X3.trainforced.qwen3-1.7b.gv.warm400\endcsname{$8.8$}
\expandafter\def\csname npj@PH.X3.trainforced.qwen3-1.7b.mp.noanchor\endcsname{$41.4$}
\expandafter\def\csname npj@PH.X3.trainforced.qwen3-1.7b.mp.noanchor_refrat\endcsname{$26.9$}
\expandafter\def\csname npj@PH.X3.trainforced.qwen3-1.7b.mp.warm400\endcsname{$23.8$}
\expandafter\def\csname npj@PH.X4.five_seeds_gv.qwen3-0.6b.gv.acc.cipp\endcsname{$-3.4$ to $-0.5$}
\expandafter\def\csname npj@PH.X4.five_seeds_gv.qwen3-0.6b.gv.acc.hierpp\endcsname{$-4.7$ to $0.9$}
\expandafter\def\csname npj@PH.X4.five_seeds_gv.qwen3-0.6b.gv.acc.p\endcsname{$P=0.016$}
\expandafter\def\csname npj@PH.X4.five_seeds_gv.qwen3-0.6b.gv.acc.pp\endcsname{$-2.0$}
\expandafter\def\csname npj@PH.X4.five_seeds_gv.qwen3-0.6b.gv.acc.ppabs\endcsname{$2.0$}
\expandafter\def\csname npj@PH.X4.five_seeds_gv.qwen3-0.6b.gv.acc.seeds\endcsname{4 of 5}
\expandafter\def\csname npj@PH.X4.five_seeds_gv.qwen3-0.6b.gv.aurc.diff\endcsname{$-0.091$}
\expandafter\def\csname npj@PH.X4.five_seeds_gv.qwen3-0.6b.gv.aurc.hier\endcsname{$-0.119$ to $-0.062$}
\expandafter\def\csname npj@PH.X4.five_seeds_gv.qwen3-0.6b.gv.aurc.p\endcsname{$P<0.001$}
\expandafter\def\csname npj@PH.X4.five_seeds_gv.qwen3-0.6b.gv.aurc.seeds\endcsname{5 of 5}
\expandafter\def\csname npj@PH.X4.five_seeds_gv.qwen3-0.6b.gv.brier.diff\endcsname{$-0.016$}
\expandafter\def\csname npj@PH.X4.five_seeds_gv.qwen3-0.6b.gv.brier.p\endcsname{$P=0.059$}
\expandafter\def\csname npj@PH.X4.five_seeds_svampv.qwen3-0.6b.svampv.acc.cipp\endcsname{$-6.2$ to $-1.8$}
\expandafter\def\csname npj@PH.X4.five_seeds_svampv.qwen3-0.6b.svampv.acc.hierpp\endcsname{$-7.7$ to $-0.4$}
\expandafter\def\csname npj@PH.X4.five_seeds_svampv.qwen3-0.6b.svampv.acc.p\endcsname{$P<0.001$}
\expandafter\def\csname npj@PH.X4.five_seeds_svampv.qwen3-0.6b.svampv.acc.pp\endcsname{$-4.0$}
\expandafter\def\csname npj@PH.X4.five_seeds_svampv.qwen3-0.6b.svampv.acc.ppabs\endcsname{$4.0$}
\expandafter\def\csname npj@PH.X4.five_seeds_svampv.qwen3-0.6b.svampv.acc.seeds\endcsname{5 of 5}
\expandafter\def\csname npj@PH.X4.five_seeds_svampv.qwen3-0.6b.svampv.aurc.diff\endcsname{$-0.113$}
\expandafter\def\csname npj@PH.X4.five_seeds_svampv.qwen3-0.6b.svampv.aurc.p\endcsname{$P<0.001$}
\expandafter\def\csname npj@PH.X4.five_seeds_svampv.qwen3-0.6b.svampv.aurc.seeds\endcsname{5 of 5}
\expandafter\def\csname npj@PH.X4.mean.gv.grpo.acc\endcsname{$81.4$}
\expandafter\def\csname npj@PH.X4.mean.gv.warm400.acc\endcsname{$79.4$}
\expandafter\def\csname npj@PH.X4.seed41.gv.acc\endcsname{$+2.9$}
\expandafter\def\csname npj@PH.X4.seed41.svampv.acc\endcsname{$-0.1$}
\expandafter\def\csname npj@PH.X4.seed53.gv.acc\endcsname{$-4.7$}
\expandafter\def\csname npj@PH.X4.seed53.svampv.acc\endcsname{$-9.6$}
\expandafter\def\csname npj@Q3a_vs_grpo.phi-4-mini.mp.oracle_cov@0.2.pp\endcsname{$+38.2$}
\expandafter\def\csname npj@Q3a_vs_grpo.phi-4-mini.mp.oracle_cov@0.2.ppabs\endcsname{$38.2$}
\expandafter\def\csname npj@Q3a_vs_grpo.qwen3-0.6b.mp.oracle_cov@0.2.pp\endcsname{$+12.3$}
\expandafter\def\csname npj@Q3a_vs_grpo.qwen3-0.6b.mp.oracle_cov@0.2.ppabs\endcsname{$12.3$}
\expandafter\def\csname npj@Q3b.qwen3-0.6b.svampv.aurc.diff\endcsname{$-0.137$}
\expandafter\def\csname npj@Q3b.qwen3-1.7b.svampv.aurc.diff\endcsname{$-0.023$}
\expandafter\def\csname npj@Q3b.sec.qwen3-0.6b.svampv.acc.pp\endcsname{$-3.4$}
\expandafter\def\csname npj@arm.qwen3-0.6b.gv.fork_grpo.kl_tokens\endcsname{$4.04$}
\expandafter\def\csname npj@arm.qwen3-0.6b.gv.fork_grpo.mean_rationale_tokens\endcsname{$152$}
\expandafter\def\csname npj@arm.qwen3-0.6b.gv.fork_grpo.satwrong\endcsname{$71.0$}
\expandafter\def\csname npj@arm.qwen3-0.6b.gv.fork_grpo.single.ts.acc.pct\endcsname{$83.0$}
\expandafter\def\csname npj@arm.qwen3-0.6b.gv.fork_rl.kl_tokens\endcsname{$2.04$}
\expandafter\def\csname npj@arm.qwen3-0.6b.gv.fork_rl.mean_rationale_tokens\endcsname{$146$}
\expandafter\def\csname npj@arm.qwen3-0.6b.gv.fork_rl.single.ts.acc.pct\endcsname{$80.9$}
\expandafter\def\csname npj@arm.qwen3-0.6b.gv.grpo.single.ts.acc.pct\endcsname{$81.9$}
\expandafter\def\csname npj@arm.qwen3-0.6b.gv.noanchor.single.ts.acc.pct\endcsname{$83.0$}
\expandafter\def\csname npj@arm.qwen3-0.6b.gv.s1cont.kl_tokens\endcsname{$0.92$}
\expandafter\def\csname npj@arm.qwen3-0.6b.gv.s1cont.mean_rationale_tokens\endcsname{$146$}
\expandafter\def\csname npj@arm.qwen3-0.6b.gv.s1cont.satwrong\endcsname{$6.0$}
\expandafter\def\csname npj@arm.qwen3-0.6b.gv.s1cont.single.ts.acc.pct\endcsname{$78.7$}
\expandafter\def\csname npj@arm.qwen3-0.6b.gv.thinkeval_grpo.length_limit_rate\endcsname{$12.5$}
\expandafter\def\csname npj@arm.qwen3-0.6b.gv.thinkeval_warm400.length_limit_rate\endcsname{$25.7$}
\expandafter\def\csname npj@arm.qwen3-0.6b.gv.warm400.single.ts.acc.pct\endcsname{$79.2$}
\expandafter\def\csname npj@arm.qwen3-0.6b.mp.fork_grpo.satwrong\endcsname{$43.3$}
\expandafter\def\csname npj@arm.qwen3-0.6b.mp.s1cont.satwrong\endcsname{$1.0$}
\expandafter\def\csname npj@arm.qwen3-0.6b.mp.thinkeval_grpo.length_limit_rate\endcsname{$35.5$}
\expandafter\def\csname npj@arm.qwen3-0.6b.mp.thinkeval_warm400.length_limit_rate\endcsname{$41.9$}
\expandafter\def\csname npj@arm.qwen3-0.6b.mp.warm400.satwrong\endcsname{$0.3$}
\expandafter\def\csname npj@arm.qwen3-1.7b.chaosnli.cn_s1_lam0.dispersion\endcsname{$0.003$}
\expandafter\def\csname npj@arm.qwen3-1.7b.chaosnli.cn_s1_lam1.dispersion\endcsname{$0.224$}
\expandafter\def\csname npj@arm.qwen3-1.7b.chaosnli.cn_s2_rl.dispersion\endcsname{$0.002$}
\expandafter\def\csname npj@arm.qwen3-1.7b.gv.fork_grpo.single.ts.acc.pct\endcsname{$93.6$}
\expandafter\def\csname npj@arm.qwen3-1.7b.gv.noanchor.satwrong\endcsname{$30.9$}
\expandafter\def\csname npj@arm.qwen3-1.7b.gv.noanchor.single.T\endcsname{$1.51$}
\expandafter\def\csname npj@arm.qwen3-1.7b.gv.s1cont.single.ts.acc.pct\endcsname{$93.5$}
\expandafter\def\csname npj@arm.qwen3-1.7b.gv.warm400.satwrong\endcsname{$92.3$}
\expandafter\def\csname npj@arm.qwen3-1.7b.gv.warm400.single.T\endcsname{$7.19$}
\expandafter\def\csname npj@arm.qwen3-1.7b.mp.grpo.satwrong\endcsname{$68.5$}
\expandafter\def\csname npj@arm.smollm3-3b.mp.grpo.satwrong\endcsname{$60.6$}
\expandafter\def\csname npj@arm.smollm3-3b.mp.warm400.satwrong\endcsname{$1.2$}
\expandafter\def\csname npj@klratio.max\endcsname{$0.43$}
\expandafter\def\csname npj@klratio.min\endcsname{$0.21$}
\expandafter\def\csname npj@rc.qwen3-0.6b.mp.warm400.0.2.ncert\endcsname{1 of 3}
\expandafter\def\csname npj@run.qwen3-1.7b.chaosnli.cn_s2_correct.23.single.raw.jsd\endcsname{$0.0582$}
\expandafter\def\csname npj@run.qwen3-1.7b.chaosnli.cn_s2_rl.23.single.raw.jsd\endcsname{$0.0644$}
\expandafter\def\csname npj@run.qwen3-1.7b.mp.credo.23.tokens\endcsname{$3.9$}
\expandafter\def\csname npj@stage2_vs_stage1.qwen3-0.6b.gv.acc.p\endcsname{$P=0.28$}
\expandafter\def\csname npj@stage2_vs_stage1.qwen3-0.6b.gv.acc.pp\endcsname{$+1.7$}
\expandafter\def\csname npj@trainlog.qwen3-0.6b.gv.credo.17.step20.conf\endcsname{$0.606$}
\expandafter\def\csname npj@trainlog.qwen3-0.6b.gv.credo.17.step40.acc\endcsname{$0.784$}
\expandafter\def\csname npj@trainlog.qwen3-0.6b.gv.credo.17.step40.conf\endcsname{$1.000$}
\makeatother

\begin{document}
\twocolumn[
\icmltitle{Beyond Outcome Rewards: Post-Rationale Distribution Optimization for Selective Prediction in Reasoning Models}
\begin{icmlauthorlist}
\icmlauthor{Zhimin Gao}{rit}
\icmlauthor{Meibo Hu}{rit}
\icmlauthor{Jiamian Wang}{rit}
\icmlauthor{Zhiqiang Tao}{rit}
\icmlauthor{Pichao Wang}{nvidia}
\end{icmlauthorlist}
\icmlaffiliation{rit}{Rochester Institute of Technology}
\icmlaffiliation{nvidia}{NVIDIA}
\icmlcorrespondingauthor{Zhimin Gao}{zhimingao113@gmail.com}
\icmlcorrespondingauthor{Pichao Wang}{pichaowang@gmail.com}
\icmlkeywords{reasoning models, selective prediction, uncertainty, calibration, reinforcement learning}
\vskip 0.3in
]
\printAffiliationsAndNotice{}
\begin{abstract}
Reasoning models are trained to produce correct answers, yet selective deployment requires ranking which
answers to accept or defer, and outcome-reward training does not directly optimize the post-rationale distribution
used for that decision. We identify two mismatches: expected correctness is mode-seeking rather than
distribution-eliciting, and scoring an average across rationales rewards disagreement unavailable to a
single-rationale decision. We introduce Post-Rationale Distribution Optimization (PRDO), which applies a proper score
to the answer distribution after each sampled rationale, without target rationales or separate confidence labels.
Across four 0.6--3.8-billion-parameter models on MMLU-Pro, PRDO improves Brier score, error ranking and secondary
accuracy relative to budget-matched, temperature-scaled GRPO. At 20\% realized error, oracle coverage rises by
\npj{Q3a_vs_grpo.qwen3-0.6b.mp.oracle_cov@0.2.ppabs}--\npj{Q3a_vs_grpo.phi-4-mini.mp.oracle_cov@0.2.ppabs} percentage
points, and ranking gains transfer to unseen tasks. Matched controls identify proper, per-rationale scoring as the
active design choices, while verification tasks show the boundary: better selective ranking can coexist with lower
full-coverage accuracy. Code is available at \url{https://github.com/ZimmyGao/openjev-rlcd}.
\end{abstract}

\section{Introduction}

A reasoning model with the option to defer must solve two problems: produce an answer and decide whether that answer
should be accepted. Accuracy measures the first but not the second. Two models can solve the same fraction of
questions yet support very different decisions if one assigns its highest confidence to its mistakes. For selective
prediction, the operational output is therefore the probability distribution produced after reasoning: its argmax
determines the answer, and its maximum determines whether the answer is accepted~\citep{geifman2017selective,
geifman2019bias}.

Most reasoning models are trained on a different object. Reinforcement learning from verifiable rewards reinforces
sampled trajectories that end in a correct answer~\citep{shao2024deepseekmath}. At the answer readout, the expected
reward is linear in the reported probabilities and is maximized by concentrating on the most likely outcome. A
selective decision instead depends on preserving distinctions between more and less reliable answers. Temperature
scaling can correct a global confidence level after training, but it cannot repair a binary confidence ranking and
need not repair a multiclass one~\citep{guo2017calibration}. The dominant training objective and the deployed
decision therefore act on different probabilistic objects.

Calibration-aware methods often add another output. Reinforcement learning with calibration rewards (RLCR) trains a
numerical confidence statement written after the answer~\citep{damani2025beyond}; related methods reward verbalized
confidence or decouple it from answer generation~\citep{baniharouni2026rewarding,ma2026decoupling,wang2026capo}.
CREDO instead reads confidence from a dedicated slot after a self-analysis~\citep{fan2026credo}. We take a more direct
route: train the answer distribution that already makes the decision. We call this
\emph{Post-Rationale Distribution Optimization} (PRDO). After sampling a rationale, PRDO reads the model's
next-token distribution over the candidate answers and optimizes it with a strictly proper scoring rule
~\citep{brier1950verification,gneiting2007strictly}. It requires neither target rationales nor verbalized confidence;
a token-level divergence penalty keeps rationale generation near the initial model. PRDO is therefore a change in
the object of optimization, not an auxiliary confidence channel.

This choice resolves two training--decision mismatches. First, a proper score elicits an outcome distribution,
whereas expected sampled-answer correctness is mode-seeking. Second, scoring an average over several rationales
rewards disagreement among their readouts, although a system that deploys one sampled rationale never observes that
average. PRDO instead scores each post-rationale readout separately. We test these consequences with matched controls:
a dense linear score has the same differentiable access to the answer probabilities but retains the mode-seeking
optimum, and a mixture objective sees the same rationales but scores only their average.

The evidence is organized around one question: does training the deployed distribution produce better decisions?
The central comparison uses four 0.6--3.8-billion-parameter models on MMLU-Pro against budget-matched group relative
policy optimization (GRPO), with temperature scaling applied to every method. We then test mechanism through the
proper-score and per-rationale controls, utility through fixed-risk coverage and transfer, and scope through
mathematical answer verification, longer thinking and rationale-reinforcement experiments.

PRDO improves probability quality and error ranking on all four MMLU-Pro backbones, increases secondary accuracy and
supports \npj{Q3a_vs_grpo.qwen3-0.6b.mp.oracle_cov@0.2.ppabs}--\npj{Q3a_vs_grpo.phi-4-mini.mp.oracle_cov@0.2.ppabs}
percentage points more oracle coverage at 20\% realized error. The ranking advantage transfers to two unseen tasks,
while the matched controls identify proper, per-rationale scoring as the operative choices. Verification tasks reveal
the boundary: selective ranking can improve even when full-coverage accuracy falls. The contribution is thus not a
claim that PRDO universally improves reasoning. It isolates and addresses a distinct learning problem---training the
distribution from which a reasoning model decides whether its own answer should be used.

\section{Selective Prediction Exposes Two Objective Mismatches}

Let $x$ denote a question, $r$ a rationale sampled from a reasoning model and
$u_\theta(x,r)\in\Delta_K$ the model's next-token distribution over $K$ candidate answers after that rationale. A
single-query system chooses $\arg\max_k u_{\theta,k}(x,r)$ and uses $\max_k u_{\theta,k}(x,r)$ as its confidence
(Fig.~\ref{fig:schematic}a). The same vector therefore determines both what the system answers and whether the answer
is accepted.

Consider an outcome $Y\sim q(\cdot\mid x)$ and an answer $A\sim u$. The expected correctness reward of the sampled
answer is
\begin{equation*}
R_{\mathrm{corr}}(u;q)=\mathbb E\,\mathbf 1[A=Y]=u^\top q.
\end{equation*}
It is maximized by any distribution supported on a mode of $q$; when the mode is unique, certainty on that mode is
optimal. By contrast, a strictly proper score $J$ satisfies
\begin{equation*}
\mathbb E_{Y\sim q}J(q,Y)\geq \mathbb E_{Y\sim q}J(u,Y),
\end{equation*}
with equality only for $u=q$~\citep{gneiting2007strictly}. Correctness and proper scoring therefore supply different
signals whenever the outcome is not fully resolved by the model's representation (Fig.~\ref{fig:schematic}c). On a
dataset with deterministic answer keys and unrestricted capacity, both objectives share the one-hot optimum. Their
distinction is an optimization and finite-capacity question: whether training pushes unresolved cases toward a
useful distribution or prematurely toward certainty.

The unit of supervision creates a second mismatch. Let $u_i=u_\theta(x,r_i)$ be the readouts after $M$ sampled
rationales and $\bar u=M^{-1}\sum_i u_i$. For the Brier score
$J(u,q)=2u^\top q-\lVert u\rVert^2$, the ambiguity identity gives
\begin{equation*}
J(\bar u,q)=\frac{1}{M}\sum_{i=1}^{M}J(u_i,q)
+\frac{1}{M}\sum_{i=1}^{M}\lVert u_i-\bar u\rVert^2.
\end{equation*}
Scoring the mixture thus includes a positive disagreement term. A set of individually poor readouts can score well if
their errors cancel, although a deployment that samples one rationale observes only one of them. The objective that
matches this deployment scores every $u_i$ separately. Together, these mismatches identify a single training target:
every post-rationale readout on which the deployed system may act.

\section{PRDO Aligns Training with the Deployed Readout}

Post-Rationale Distribution Optimization (PRDO) implements these two choices directly (Fig.~\ref{fig:schematic}b). For each question it
samples $M=4$ rationales from the current model, reads the answer distribution after each rationale and minimizes
\begin{equation*}
\begin{aligned}
\mathcal L_{\mathrm{PRDO}}(\theta)
={}&-\frac{1}{M}\sum_{i=1}^{M}J\big(u_\theta(x,r_i),y\big)\\
&+\frac{\lambda}{M}\sum_{i=1}^{M}\sum_t D_{i,t},\\
D_{i,t}={}&\mathrm{KL}\!\left(\pi_\theta(\cdot\mid h_{i,t})\,\middle\|\,
\pi_0(\cdot\mid h_{i,t})\right).
\end{aligned}
\end{equation*}
Here $h_{i,t}=(x,r_{i,<t})$ is the token history, $J$ is the Brier or log score and $\pi_0$ is the initial model. Gradients pass through the answer readout and the
token-level anchor, with all shared weights trainable; the sampled rationale itself receives no reinforcement-learning
reward. The method consequently needs an answer label but neither a target rationale nor a separate confidence
annotation.

\begin{figure*}[t]
\centering
\includegraphics[width=\textwidth]{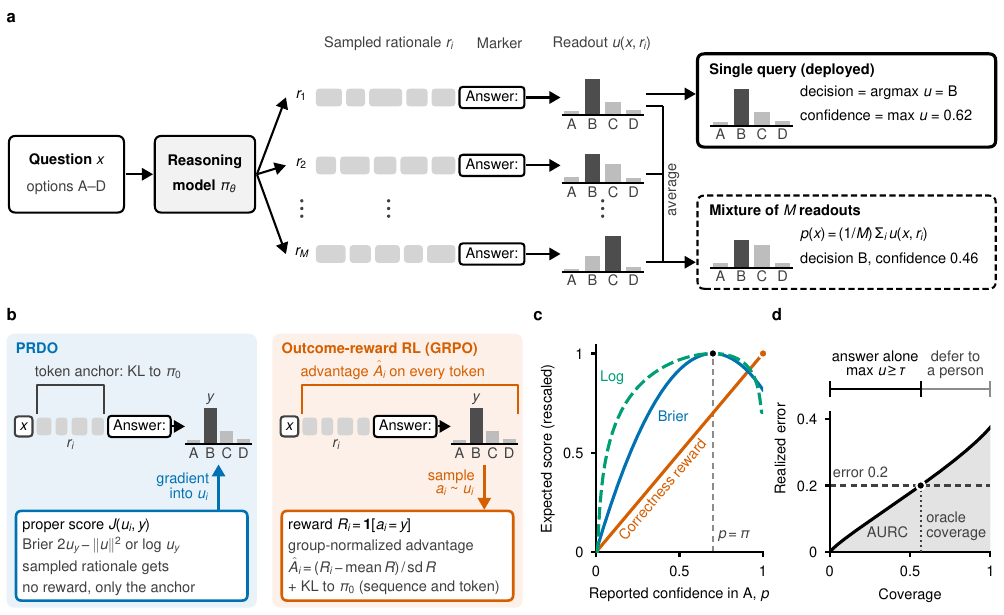}
\caption{\textbf{PRDO aligns training with the deployed answer distribution.}
\textbf{a}, A model samples a rationale $r_i$ and produces a distribution $u(x,r_i)$ over candidate answers at the
answer marker. Its argmax is the decision and its maximum is the confidence; a deployed single query exposes one such
readout. \textbf{b}, Post-Rationale Distribution Optimization (PRDO) differentiates a proper score $J(u_i,y)$ through each readout and uses
a token-level Kullback--Leibler (KL) penalty to anchor the rationale to the initial model. GRPO instead samples an
answer and reinforces the rationale and answer according to correctness. Both use four rationales per question at
matched budget. \textbf{c}, For a binary outcome that is correct with probability 0.7, proper Brier and log scores are
maximized at the outcome probability, whereas expected sampled-answer correctness is maximized at certainty. Curves
are exact expectations rescaled to $[0,1]$. \textbf{d}, Selective prediction accepts answers above a confidence
threshold. The risk--coverage curve records realized error as coverage increases; its area is AURC. Panels a, b and d
are schematics; d uses simulated confidences.}
\label{fig:schematic}
\end{figure*}

\section{Experimental Setup}\label{sec:experimental-setup}

We evaluated the Brier form against group relative policy optimization followed by temperature scaling (GRPO+TS).
GRPO samples an answer from the same readout and reinforces the rationale and sampled answer according to correctness.
The two methods used the same initial model, questions, number of updates and four rationales per question. The central
task was MMLU-Pro on Qwen3-0.6B, Qwen3-1.7B, SmolLM3-3B and Phi-4-mini; GSM8K-Verify on the two Qwen3 models tested
mathematical answer verification. Each trained arm used three seeds. Every method was temperature-scaled on development
questions, and final comparisons used held-out test questions.

The primary outcomes capture complementary properties of the deployed single-rationale readout. The multiclass Brier
score measures the quality of the full distribution. The area under the risk--coverage curve (AURC) measures how well
confidence ranks errors as progressively fewer answers are accepted. Accuracy is secondary: it evaluates the argmax,
whereas Brier score and AURC evaluate the distribution used to make selective decisions.

\subsection{Evaluation Design}

We specified the main data splits, training arms, model-selection rules, endpoints and statistical comparisons in an
analysis plan before the main final-test experiments. This plan was not registered in an external
repository. Pilot analyses reported in version 1 of this manuscript used an earlier test split, which was subsequently
designated as validation data; no pilot result is treated as a final-test result. Split manifests fix the evaluated items,
and each completed final run evaluated its final test set once. Limited code smoke tests had previously read four
items from the final-test and calibration splits; their scope and the timing of analysis-plan extensions are
reported in Appendix~\ref{app:provenance}. The complete plan and its revision history are provided in the archived analysis plan.

Transfer, ChaosNLI, thinking-mode and CREDO comparisons were specified before their respective training or evaluation
runs. The RLCR comparison was specified after its completed-run logs had been inspected and is labelled post hoc.
Follow-up experiments on rationale rewards, training duration, the token anchor and two additional training seeds
were designed after the main analyses and are exploratory (Appendix~\ref{app:exploratory}; the archived exploratory analysis plan).
Secondary coverage analyses computed after the main decision analysis are also labelled post hoc. We retain all
reported outcomes regardless of direction and distinguish these exploratory analyses from the pre-specified
comparisons throughout.

\subsection{Tasks and Splits}

\paragraph{MMLU-Pro.} Multiple-choice questions with up to ten options~\citep{wang2024mmlupro}. A seed-0 shuffle of the
12,032 test questions gives 4,000 training, 500 development (300 used) and 1,500 further questions, whose first 1,000
are the validation split. Of the remaining 6,032, unread by any run of our earlier analyses (only code smoke tests,
with models trained for two to four steps, evaluated the first four final-test and calibration items; Appendix~\ref{app:provenance}), those whose lower-cased,
whitespace-normalized stem
repeats no training or development question give, in order, the final test set (1,000) and the calibration set (1,000);
the validation split has 24 such repeats.

\paragraph{GSM8K-Verify.} Is a proposed final answer to a GSM8K problem~\citep{cobbe2021training} correct? Each problem
got one proposal, a brief solution sampled from Qwen3-0.6B (seed 0, temperature 1, $\leq320$ tokens), labelled by
comparing its final number with the reference; problems without a parsable number were dropped (7,144 of 7,473 training and 1,260 of 1,319 test problems remain). Qwen3-0.6B thus verifies
its own earlier answers. After a seed-0 shuffle, 4,000 GSM8K training problems are for training and 500 for development
(300 used); the 1,260 test problems with a proposal are the validation split. Of the 2,644 remaining training problems,
unused by any run of our earlier analyses (apart from a smoke test of the first four items of each split, as for
MMLU-Pro), the first 1,260 are the final test set (52.9\% correct proposals, against 49.8\% in the validation
split; the backbones may know training problems better) and the next 1,260 the calibration set. Methods are compared
only within a split. Final evaluation of GSM8K-Verify used the two Qwen3 models only; SmolLM3-3B and Phi-4-mini were trained
on it only during development (Supplementary Table~6).

\paragraph{Transfer tasks.} The Qwen3 zero-shot models and the final weights of SFT+TS, PRDO
(Brier and log score), GRPO+TS and the two-stage arm were evaluated once, with the source task's evaluation settings, on
unread tasks: MMLU-Pro models on ARC-Challenge~\citep{clark2018arc} (1,172 test and 299 development questions; MMLU-Pro
prompt and readout, options lettered in dataset order); GSM8K-Verify models on SVAMP-Verify, built once like
GSM8K-Verify (same solver, prompt, sampling and seed) from the 1,000 SVAMP problems~\citep{patel2021svamp} (917 parsable
proposals, 67.4\% of them correct; after a seed-0 shuffle, 200 development and 717 test items, 68.1\% correct), so that
Qwen3-0.6B again verifies its own answers.

\paragraph{ChaosNLI.} ChaosNLI~\citep{nie2020learn} gives 100 annotator labels for each of 1,599 MNLI
items~\citep{williams2018broad}, whose distribution $q$ is the target. A seed-0 shuffle gives 1,000 training, 199
development and 400 test items; each training visit draws one annotator label. Qwen3-1.7B was trained by PRDO
with the per-rationale or the mixture objective (three seeds each) for 800 steps with $M=8$ rationales per
item. Its test items had been evaluated in version 1 of this manuscript (ChaosNLI has no
unused items).

\subsection{Models and Decoding}

We used Qwen3-0.6B, Qwen3-1.7B~\citep{qwen3}, SmolLM3-3B~\citep{smollm3} and Phi-4-mini-instruct (3.8 billion
parameters)~\citep{phi4mini} with their chat templates, in non-thinking mode unless stated. A question is one user turn
asking for brief reasoning ending in a line `Answer: X' (prompts in Appendix~\ref{app:prompts}). Rationales are sampled
at temperature 1, without top-$p$ or top-$k$ truncation, until the marker \texttt{Answer:}, the end-of-sequence token or
a cap of 256 tokens (MMLU-Pro, ARC-Challenge), 320 (GSM8K-Verify, SVAMP-Verify) or 128 (ChaosNLI); a missing marker is
appended on a new line, after the end-of-sequence token if one was sampled (a forced marker). The readout is the next-token distribution at the marker over the first tokens of the
question's labels, each with a leading space. Evaluation samples $M=4$ rationales per question (ChaosNLI 16, thinking
mode 1); the single query is the readout after the first, the mixture readout the mean of all. Training uses fp32
weights under bf16 autocast, sampling from a bf16 copy refreshed every step (SmolLM3-3B, Phi-4-mini: a single copy with
bf16 linear layers; fp32 master weights and AdamW state partly in host memory); final weights are saved in bf16 (except the exploratory rationale-reward, continuation and anchor-ablation runs).

\subsection{Readout and Training Objectives}

For a question $x$ with $K$ options and a rationale $r$, the readout
$u_\theta(x,r)=\mathrm{softmax}(z_\theta(x,r))\in\Delta_K$, with $z_\theta$ the logits of the $K$ option tokens at the
marker, is read, not sampled; its arg max is the decision, its maximum the confidence. A single query reports
$u_\theta(x,r)$ for one rationale $r\sim\pi_\theta(\cdot\mid x)$; the mixture is
$p_\theta(x)=\mathbb{E}_r\,u_\theta(x,r)$. As rewards, the Brier score is $J(u,y)=2u_y-\|u\|^2=1-\|u-e_y\|^2$ and the
log score $\log u_y$, both strictly proper~\citep{brier1950verification,gneiting2007strictly}. For a label distribution
$q$ and $J(u,q)=\mathbb{E}_{Y\sim q}J(u,Y)$, the ambiguity decomposition~\citep{krogh1995neural}
\begin{equation*}
J\big(p_\theta(x),q\big)=\mathbb{E}_r\,J\big(u_\theta(x,r),q\big)+\mathbb{E}_r\big\|u_\theta(x,r)-p_\theta(x)\big\|^2
\end{equation*}
adds the disagreement between rationales to the expected single-query score. The per-rationale objective
$\mathbb{E}_r\,J(u_\theta(x,r),q)$, which we use, is maximized by $u_\theta(x,r)=q(x)$ for every rationale, as the
outcome does not depend on the rationale given the question; the mixture objective $J(p_\theta(x),q)$, a control, is
also maximized by disagreeing readouts that average to $q$. For one-hot $q$ (MMLU-Pro, GSM8K-Verify), both objectives
and the expected correctness $u^\top q$ of an answer sampled from the readout, which outcome-reward reinforcement
learning maximizes, share the maximizer $u=e_y$, so differences between trained models reflect training at a finite
budget and capacity, not their optima.

\subsection{PRDO Training Configuration}

Post-Rationale Distribution Optimization (PRDO) draws 8 training questions per step uniformly with replacement, with outcomes $Y$ (answer keys; on ChaosNLI annotator
labels), from a generator giving all arms with the same seed the same questions and outcomes. For each question $M=4$
rationales $r_i$ are sampled from the current model, and the loss
\begin{equation*}
\begin{aligned}
\mathcal L_{\mathrm{PRDO}}(\theta)
={}&-\frac1M\sum_{i=1}^M J\big(u_\theta(x,r_i),Y\big)\\
&+\frac{\lambda}{M}\sum_{i=1}^M\sum_{t=1}^{|r_i|}D_{i,t},\\
D_{i,t}={}&\mathrm{KL}\!\left(\pi_\theta(\cdot\mid h_{i,t})\,\middle\|\,
\pi_0(\cdot\mid h_{i,t})\right),
\end{aligned}
\end{equation*}
with $h_{i,t}=(x,r_{i,<t})$ and $J$ the Brier or log score, is averaged over questions. Gradients flow through the readouts and the anchor's
next-token distributions; the rationales receive no reward. The token anchor ($\lambda=0.01$), an exact
Kullback--Leibler (KL) divergence to the frozen initial model $\pi_0$, limits how far readout training moves the
rationales, which share all weights with the readout. The configured sequence-level KL weight ($\beta=0.04$, used in stage 2) has no effect here, as stage 1 has no
score-function term. Training runs for 400 steps with AdamW~\citep{loshchilov2019decoupled} (constant learning rate
$2\cdot10^{-6}$, no weight decay, otherwise PyTorch defaults), gradient norms clipped at 1, seeds 17, 23 and 29. Without rationale or anchor, the
log-score form is SFT+TS; the Brier form is the first phase of a reasoning-free adaptation~\citep{wang2026unijev} of the method
reported in version 1 of this manuscript.

\paragraph{Controls.} At the same budget and anchor: the \emph{linear score} $2u_Y-1$, the Brier score without
$-\|u\|^2$ (as dense, but improper); \emph{offline rationales}, sampled afresh at every step from the frozen initial model instead of the
trained one (evaluation still samples the model's own); and the \emph{mixture objective}, whose first term is the
unbiased estimate of the mixture's Brier score from the $M$ rationales, $\frac2M\sum_iu_{i,Y}-\frac1{M(M-1)}\sum_{i\neq
j}u_i^\top u_j$ with $u_i=u_\theta(x,r_i)$.

\subsection{Reinforcing the Reasoning}

The optional second stage continues from the stage-1 weights $\theta_1$ for 400 steps (fresh optimizer, seed $1000+s$ for stage-1 seed $s$; 800 steps in all, twice the budget of stage 1 and of the baselines),
keeping the stage-1 loss with $\pi_{\theta_1}$ as the anchor's reference and adding a score-function term for the
rationales~\citep{williams1992simple,kool2019buy,ahmadian2024back},
\begin{equation*}
\begin{aligned}
\mathcal L_{\mathrm{SF}}
&=-c\sum_{i=1}^M\mathrm{sg}(A_i)\log\pi_\theta(r_i\mid x),\\
A_i&=R_i-\frac1{M-1}\sum_{j\neq i}R_j,\\
R_i&=\frac1M\Bigg[J\big(u_\theta(x,r_i),Y\big)\\[-2pt]
&\qquad-\beta\log\frac{\pi_\theta(r_i\mid x)}{\pi_{\theta_1}(r_i\mid x)}\Bigg].
\end{aligned}
\end{equation*}
with $c=0.3$, $\beta=0.04$ and $\mathrm{sg}$ the stop-gradient. With $c=1$ the readout and score-function terms (leaving aside the token anchor, which adds its pathwise term as in stage 1) would be an unbiased estimate of
the gradient of the expected per-rationale score minus $\beta$ times the rationale distribution's KL divergence from
$\pi_{\theta_1}$ (checked by exact enumeration on tabular models); $c=0.3$ down-weights the score-function part. Two
forks on Qwen3-1.7B start from the same checkpoints: the \emph{sampled-answer reward} replaces $J(u_\theta(x,r_i),Y)$ in
$R_i$ by $2\cdot\mathbf 1[\hat y_i=Y]-1$ for one answer $\hat y_i$ sampled from the readout; the \emph{GRPO
continuation} replaces the whole objective by that of GRPO+TS, with $\pi_{\theta_1}$ as reference. In exploratory follow-up experiments (Appendix~\ref{app:exploratory}), every Qwen3 stage-1 checkpoint was also continued for 400 steps without the score-function term
(the stage-1 continuation, a budget control), and the Qwen3-0.6B ones also with the GRPO continuation.

\subsection{Baselines}

All baselines start from the same model, see the same questions and outcomes per seed and train for 400 steps with the
same optimizer, sampling four rationales per question if they reason. Every arm is temperature-scaled (+TS in baseline
names). The \emph{zero-shot} model is evaluated like the trained arms (seed 17). \textbf{SFT+TS} starts the assistant's
turn at the marker (an empty rationale) under the same prompt, which asks for reasoning, and trains the label token with the log score. Further baselines,
rejection-sampling fine-tuning (RFT)~\citep{zelikman2022star}, REINFORCE leave-one-out (RLOO)~\citep{ahmadian2024back} and GRPO at twice
the budget, ran on the validation split only (development results in Supplementary Table~6).

\paragraph{GRPO+TS.} An answer $\hat y_i$ sampled from each readout earns $\mathbf 1[\hat
y_i=Y]$~\citep{shao2024deepseekmath}. Rewards standardized within the question's four rollouts (minus their mean, divided by their population s.d.\ plus $10^{-4}$) weight the
log-probabilities of the rationale and of the answer token, the latter taken as $\log u_{\hat y_i}$, the readout's probability renormalized over the $K$ options, without length normalization, in one on-policy update per
batch. The KL penalty subtracts from the rationale's advantage $\beta$ times its log-ratio to the initial model, centred
by a leave-one-out baseline. On the validation split (seed 17, four Qwen3 cells), $\beta\in\{0,0.01,0.04,0.1\}$ and
$\beta=0.04$ with the token anchor were ranked by mean single-query Brier score after scaling (mean AURC breaking ties
within 0.002), excluding settings that collapsed in any cell, as $\beta=0$ and 0.01 did on Qwen3-0.6B MMLU-Pro. The
selected setting, $\beta=0.04$ with the token anchor ($\lambda=0.01$), was used on every backbone (every setting's validation
scores per cell in Supplementary Table~6).

\paragraph{RLCR+TS.} Reinforcement learning with calibration rewards (RLCR) augments answer correctness with a
squared-error reward for a numerical confidence stated after the answer~\citep{damani2025beyond}. We prompt the model
to append a line `Confidence: C' after its sampled answer (at most 10 tokens); the combined reward is
$\mathbf 1[\hat y=Y]-(C-\mathbf 1[\hat y=Y])^2$, with an unparseable line assigned squared error 1. Rationale, answer
and confidence tokens form the action; the
KL penalty is that of GRPO+TS without the token anchor, as in the runs it mirrors; leave-one-out advantages beat
standardized ones on the validation split by the same rule (Supplementary Table~6). At evaluation the decision is the readout's most probable
answer and the confidence line is greedy; the readout puts $C$, clipped to $[0.001,0.999]$, on the decision and $1-C$
evenly on the other options (uniform if unparseable); a $C$ below $1/K$ is raised to just above it if $K>2$ and decides
the other answer if $K=2$.

\paragraph{CREDO+TS.} After its answer the model writes an analysis between \texttt{<analysis>} tags; its
confidence $c=P(\text{high})/(P(\text{high})+P(\text{low}))$ is read at a following \texttt{<confidence>}
slot~\citep{fan2026credo}. As published for mathematics, an answer reward $a_i+0.5F_i$ (correctness $a_i=\mathbf 1[\hat
y_i=Y]$, format $F_i$) and a confidence reward $F_i(1-(t_i-c_i)^2)$, $t_i=\frac12(a_i+\bar a)$, are each standardized
within the group; rationale and answer tokens get the answer advantage times
$w_i=\mathrm{clip}(1+0.25(|a_i-c_i|-\overline{|a-c|}),0.5,2)$, analysis tokens the confidence advantage, each rollout's
term divided by its token count, without a KL term; and 0.03 times the mean of $(c_i-t_i)^2$ trains $c$ through all
weights, with 0.03 and 0.25 ramped up linearly (our choice of shape) over 30 steps. Adaptations, fixed before any run:
answers sampled from the readout; the analysis opened for the model and capped at 128 tokens, the smallest of 64, 128,
192 and 256 within which $\geq90\%$ of the zero-shot models' sampled analyses on 100 development questions closed
(largest over the Qwen3 cells); the slot appended and read, so every rollout has a readout ($F_i=1$ if the rationale
wrote the marker itself and the analysis closed); `high' and `low' read directly, not as two added tokens; groups of
four rollouts, not eight, one update per batch; temperature 1 without top-$k$ (published: top-$k$ 50, evaluation at
temperature 0.7); our prompt wording. Evaluation follows RLCR+TS, with a greedy analysis; CREDO+TS was not tuned.
Supplementary Table~5 gives each run's forced-marker rate and confidence (top-label) Brier score with
its Murphy terms, since the multi-class Brier score of such a readout includes the even spread of $1-c$; Supplementary
Note~4 gives the reproduction details.

\subsection{Thinking Mode}

The Qwen3 templates rendered the same messages in thinking mode (Appendix~\ref{app:thinking}). The rationale is the think block plus the text after
it; the marker counts only after the block closes, and a forced marker first closes an open block. The cap rule (the
smallest multiple of 256 tokens within which 90\% of the zero-shot model's rationales on 300 development questions
finished, at most 2,048) was met in no cell, so the cap was 2,048. Single queries were evaluated, with temperatures
fitted on thinking-mode development readouts. The evaluation-only comparison evaluated the zero-shot models (seed 17) and the non-thinking
final weights of PRDO and GRPO+TS (three seeds) in thinking mode in the four Qwen3 cells; the thinking-mode training comparison
trained both methods (three seeds) in thinking mode on Qwen3-1.7B GSM8K-Verify, the cell with the fewest truncated
rationales. For computational feasibility, thinking-mode training was restricted to this cell, and the four-cell evaluation of
non-thinking-trained weights was added before any full thinking-mode experiment began. The initial scope had been
four training cells, requiring about 550 GPU-hours; the development length probe and code smoke tests preceded this
revision (Appendix~\ref{app:provenance}).

\subsection{Scoring}

The same code recomputes every arm's metrics from its saved readouts: log-probabilities, RLCR+TS's stated distributions
and CREDO+TS's slot logit gap $g$ ($\log c=\log\sigma(g)$, unclipped).

\paragraph{Temperature scaling.} Per run and readout, $T$ minimizes the negative log-likelihood (NLL) of the development
questions (300; ChaosNLI 199; one label per item, on ChaosNLI an annotator label drawn per seed and item) over 311
values evenly spaced in $\log T$ from 0.1 to 50~\citep{guo2017calibration}; $T=1$ if the NLL varies by less than
$10^{-4}$. Scaling never changes the decision nor, with two options, the ranking. On a transfer task the primary
temperature is the source run's; one fitted on the new task's development questions is also reported, and
every arm's accuracy, Brier score, AURC and oracle coverage are given under both (Supplementary Table~2).

\paragraph{Metrics.} For the single query after scaling unless stated: accuracy; the multi-class Brier score
$\sum_k(u_k-\mathbf 1[k=y])^2$ (0 to 2); NLL; the expected calibration error (ECE) over ten equal-width confidence
bins~\citep{naeini2015obtaining}; the binned Murphy decomposition~\citep{murphy1973new} (reliability, resolution,
uncertainty) of the confidence's Brier score as a forecast of correctness; and selective metrics, which rank items by
confidence, break ties by the log-margin between the two most probable options and give remaining ties their group's
mean error. The area under the risk--coverage curve (AURC)~\citep{geifman2019bias} is the mean, over coverages $k/n$, of
the error rate among the $k$ most confident items; E-AURC subtracts the AURC of a perfect ranking of the same errors.
The oracle coverage at realized error $\epsilon$ (0.05, 0.1, 0.2) is the largest share of top-ranked items with an error
rate of at most $\epsilon$; it uses the test labels, so it measures the ranking. The label-free coverage at
predicted error $\epsilon$ (0.1, 0.2) is the largest such share whose mean predicted error (one minus the confidence) is at most
$\epsilon$; it, its realized error and the Murphy terms are given per run in the archived run-level results. Descriptive diagnostics include the
dispersion $\|u_1-\bar u\|^2$ between raw single-query and mixture readouts (mean over items), the share of wrong
answers with unscaled confidence $\geq0.99$, the final-test rationale statistics (mean length in tokens, shares at the cap and with a forced marker, repetition ratio) and training statistics logged every 10 steps, among them the token KL, the anchor's per-rationale sum of next-token KL divergences to its reference model (the initial model; in second-stage runs, the stage-1 weights), computed only in runs with the anchor (GRPO+TS carries it too, RLCR+TS and CREDO+TS do not; reported as the mean over logged steps and seeds; Supplementary Table~3).

\paragraph{Risk-controlled thresholds.} The final runs of the zero-shot model, SFT+TS, PRDO, the
two-stage arm and GRPO+TS also evaluated the calibration set in their final-test pass, with the same temperature. From a
grid fixed in advance (0.30, 0.35, \dots, 0.95, 0.97, 0.98, 0.99, 0.995, 0.999; 19 values), the threshold is the lowest
whose accepted calibration items (scaled single-query confidence at least the threshold) have a one-sided $(1-\delta/19)$ Clopper--Pearson upper bound~\citep{clopper1934use} on
their error rate of at most $\epsilon$ (0.1 or 0.2), with $\delta=0.1$. By Bonferroni's inequality, as in fixed-grid
risk control~\citep{geifman2017selective,angelopoulos2021learn}, with probability at least 0.9 every grid threshold that
meets this bound, the chosen one included, has a selective error of at most $\epsilon$ in the population of the calibration and test items. Final-test coverage and
realized error are descriptive.

\paragraph{ChaosNLI endpoints.} Per item, the Jensen--Shannon divergence (JSD)~\citep{lin1991divergence}, in bits, between the
unscaled readout and $q$ and the squared distance $\|u-q\|^2$ (the expected Brier score against one random annotator, up
to a constant), for the single query and the 16-rationale mixture. The pre-specified comparison tested whether per-rationale scoring improves the unscaled single-query output in both
metrics (Holm over the two readouts $\times$ two metrics); greater dispersion under the mixture objective was a
descriptive prediction. Secondarily: both metrics after scaling, and majority
accuracy.

\subsection{Exclusion Rule}

A run whose test rationales collapsed (mean length $\leq5$ tokens, or $\geq99\%$ at the cap with a repetition ratio, one
minus the share of distinct words, $\geq0.9$) is left out of its arm, as is a transfer run whose source run collapsed.
The rule removed one final run, CREDO+TS on Qwen3-1.7B MMLU-Pro (seed 23; rationales of a few tokens). ChaosNLI has no
exclusion (a reasoning-free readout is legitimate there). Crashed runs are rerun with the identical command and seed (training on the GPU is not bit-reproducible, so a rerun
is a fresh draw). No final run ended with
an error, but four were interrupted and rerun from the start with the same command: two main runs (Qwen3-0.6B
GSM8K-Verify, linear score and mixture objective, seed 17), by a machine restart, and two
exploratory runs, one most likely killed by the operating system and one by another machine restart
(Appendix~\ref{app:exploratory}).

\subsection{Statistical Analysis}

Comparisons are paired over test items ($n=1{,}000$ MMLU-Pro and 1,260 GSM8K-Verify final-test questions; 1,172
ARC-Challenge, 717 SVAMP-Verify and 400 ChaosNLI items) with three seeds per arm (five in the exploratory verification follow-up; two for CREDO+TS on Qwen3-1.7B
MMLU-Pro; one zero-shot run). Differences are first arm minus second; $\pm$ is the s.d. over an arm's kept seeds. For
accuracy, Brier score, NLL and the ChaosNLI metrics, per-item values are averaged over kept seeds; the 95\% confidence
interval (CI) is the percentile interval of 20,000 item-bootstrap resamples~\citep{efron1993introduction}, and $P$ is
from a two-sided sign-flip permutation test with 20,000 random sign vectors. AURC, E-AURC and oracle coverage are
computed per seed on the same resampled items and averaged over seeds; 4,000 item-bootstrap resamples give the CI and a
two-sided $P$ (twice the smaller share of resampled differences on either side of zero). Monte Carlo $P$ values are
floored at $1/(B+1)$ for $B$ draws; a Holm-adjusted $P$ computed from a floored value is an upper bound, written $P\leq$, and CIs are not adjusted for multiplicity. We also report a hierarchical bootstrap over seeds within arms and over items (2,000
resamples) and the number of seeds agreeing in sign (second-stage runs paired with their stage-1 runs).

\paragraph{Multiple comparisons and evidence summaries.} We apply Holm correction~\citep{holm1979simple}
within separately specified families for each scientific comparison. The central family comprises four MMLU-Pro
and two GSM8K-Verify cells, each evaluated on Brier score and AURC. The model-summary criterion requires improvement
in both metrics on at least three MMLU-Pro backbones and deterioration in neither on any; GSM8K-Verify does not
determine this summary. Each proper-score and rationale-source comparison groups the four Qwen3 cells and two
primary metrics. The dense-score control groups four AURC tests; per-rationale versus mixture scoring groups the
four Brier-score tests separately for single and averaged readouts. Coverage is grouped by comparator across the
four MMLU-Pro backbones, and transfer across the four source-model/task cells. Additional rationale-reward and
closest-method comparisons have their own families. Accuracy, NLL and E-AURC are secondary endpoints with separate
families, as are the lower-risk coverage, target-temperature transfer and thinking-mode analyses. Full family
membership and the fixed cross-model evidence criteria are given in Appendix~\ref{app:statistics} and the archived analysis plan.
Unadjusted confidence intervals and seed-robustness diagnostics accompany the adjusted tests; nonsignificance is
not interpreted as equivalence.

\subsection{Compute and Software}

Each run used one of two NVIDIA RTX 6000 Ada GPUs (48 GB; 125 GB host memory), for about 0.7--2.3 h (Qwen3), 3--4.5 h
(SmolLM3-3B, Phi-4-mini) or 10 h (thinking mode) per 400 steps with rationales, plus evaluation. We used PyTorch
2.6~\citep{paszke2019pytorch} and Transformers 5.17~\citep{wolf2020transformers} (with tokenizers 0.23.2 and Datasets 5.0.1; PyTorch built for CUDA 12.4). Code and dataset versions were recorded by cryptographic hashes. A small-model regression check reproduced the
earlier implementation's saved outputs and weights in 28 configurations; version records accompany the archived
code and data.

\section{Results}\label{sec:results}

\subsection{PRDO Improves Decisions Across Model Families}

At matched training budget, PRDO improved both primary outcomes on every MMLU-Pro backbone
(Fig.~\ref{fig:central}a,b). Relative to GRPO+TS, Brier-score differences ranged from
\npj{C.qwen3-1.7b.mp.brier.diff} to \npj{C.phi-4-mini.mp.brier.diff}, and AURC differences from
\npj{C.qwen3-0.6b.mp.aurc.diff} to \npj{C.phi-4-mini.mp.aurc.diff}; all eight comparisons were significant after Holm
correction and had the same sign in every seed. The result was not restricted to probability calibration. Excess AURC,
which removes the ranking improvement attributable solely to a lower error rate, was also lower on every backbone
(differences \npj{C.sec.qwen3-1.7b.mp.eaurc.diff} to \npj{C.sec.phi-4-mini.mp.eaurc.diff}). Thus, PRDO changed which
errors received the highest confidence, not only the number of errors or the global
confidence scale.

Accuracy also increased on all four backbones, by \npj{C.sec.qwen3-1.7b.mp.acc.pp} to
\npj{C.sec.phi-4-mini.mp.acc.pp} percentage points (Fig.~\ref{fig:central}c), although the seed-level evidence was
weaker than for the two distributional outcomes. A direct-answer SFT baseline lay between PRDO
and GRPO+TS in mean Brier score and AURC on every backbone (Supplementary Table~2). These comparisons locate the main
effect in the post-rationale answer distribution rather than in temperature scaling: all methods were scaled, yet
scaling could not recover GRPO's error ranking.

The verification task exposed a different regime. On Qwen3-0.6B GSM8K-Verify, PRDO lowered AURC
(\npj{C.qwen3-0.6b.gv.aurc.diff}, Holm-adjusted \npj{C.qwen3-0.6b.gv.aurc.p}) but reduced accuracy by
\npj{C.sec.qwen3-0.6b.gv.acc.ppabs} percentage points; Brier score did not differ. On Qwen3-1.7B, none of the three
outcomes differed. Distribution supervision therefore improved selective decisions consistently on knowledge
questions but did not uniformly improve the underlying answer decision.

\begin{figure*}[t]
\centering
\includegraphics[width=\textwidth]{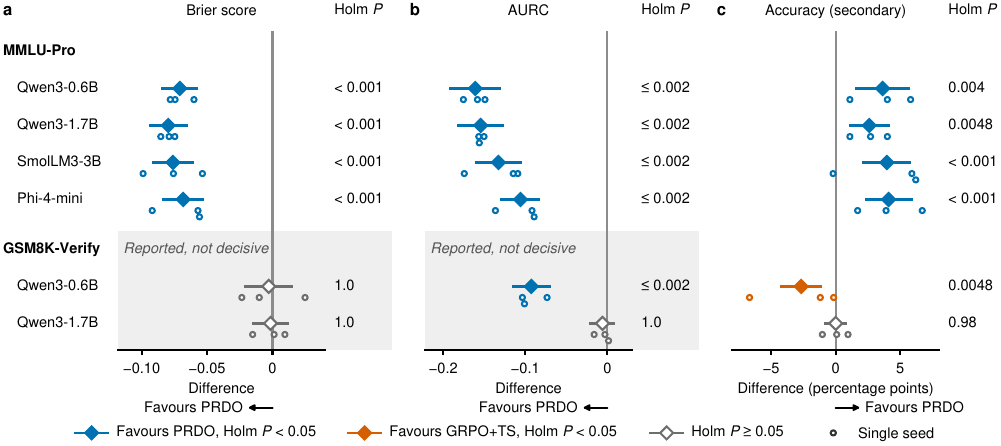}
\caption{\textbf{PRDO improves the deployed readout across four MMLU-Pro backbones.}
Paired test-set differences between PRDO (Brier score) and budget-matched GRPO followed by
temperature scaling (TS); every method is evaluated after TS. \textbf{a}, Multiclass Brier score. \textbf{b}, Area
under the risk--coverage curve (AURC). \textbf{c}, Accuracy, a secondary endpoint. Negative Brier-score and AURC
differences and positive accuracy differences favour PRDO. Rows show four models on MMLU-Pro
($n=1{,}000$ questions) and two Qwen3 models on GSM8K-Verify ($n=1{,}260$). Diamonds are differences between arm
means over three seeds; horizontal bars are unadjusted 95\% item-bootstrap confidence intervals; circles are
seed-matched differences. Filled diamonds denote Holm-adjusted $P<0.05$. Grey bands mark the verification cells,
which were reported but did not determine the central cross-model criterion. Tests and multiplicity correction are
specified in Section~\ref{sec:experimental-setup}.}
\label{fig:central}
\end{figure*}

\subsection{Proper Scores and Per-Rationale Targets Explain the Gain}

To separate proper scoring from the mere availability of a dense gradient, we replaced the Brier score by
$2u_y-1$. This linear score is the exact expected $\{-1,+1\}$ correctness reward of an answer sampled from $u$, so it
is differentiable through the same readout but remains mode-seeking. The Brier objective produced lower AURC in three
of four Qwen3 model--task cells, with agreement across all three seeds in those cells
(Fig.~\ref{fig:ingredients}c). Accuracy did not differ in any cell. The linear-score models also required a higher
fitted temperature in every cell (Fig.~\ref{fig:ingredients}d), showing that a dense gradient alone did not prevent
the readout from becoming overconfident.

The raw predictions reveal the corresponding mechanism (Fig.~\ref{fig:ingredients}e). On MMLU-Pro, GRPO assigned at
least 0.99 confidence to \npj{arm.smollm3-3b.mp.grpo.satwrong}--\npj{arm.qwen3-1.7b.mp.grpo.satwrong}\% of its wrong
answers, compared with \npj{arm.qwen3-0.6b.mp.warm400.satwrong}--\npj{arm.smollm3-3b.mp.warm400.satwrong}\% under
the Brier objective. The linear score showed the same tendency toward saturation. Temperature scaling softened these
probabilities but left their ordering largely unchanged, explaining why GRPO's calibration level could improve while
its AURC remained higher. On Qwen3-1.7B GSM8K-Verify the Brier readout also saturated; there the log score saturated
less and outperformed it, marking a limit of the Brier loss rather than of proper scoring itself
(Fig.~\ref{fig:ingredients}a).

Changing the source of the rationales did not explain the central result. Training after rationales sampled online
from the current model was indistinguishable in Brier score and AURC from training after rationales sampled from the
frozen initial model (Fig.~\ref{fig:ingredients}b). Under the token anchor, readout learning therefore mattered more
than the modest change in the rationale distribution.

The per-rationale control tested the second objective choice. With all other settings matched, scoring each readout
instead of the four-rationale average lowered the single-query Brier score in three of four cells
(Fig.~\ref{fig:mixture}a). The mixture objective produced greater dispersion among a question's readouts in every cell
(Fig.~\ref{fig:mixture}c), as predicted by its disagreement term. Unexpectedly, per-rationale scoring also gave a
lower Brier score for the averaged readout in all four cells (Fig.~\ref{fig:mixture}b), indicating that disagreement
was not merely redistributed into a better ensemble at this training budget.

ChaosNLI provides a setting in which the target is explicitly a distribution over 100 annotators. Here a mixture can
match the target while its individual readouts disagree. Per-rationale supervision made those readouts nearly
identical and moved both the single and 16-rationale distributions closer to the annotator distribution
(Fig.~\ref{fig:mixture}d). This replication on previously read test items supports the mechanism suggested by the
decomposition: the useful unit of supervision is the same unit that the system will expose at deployment.

\begin{figure*}[t]
\centering
\includegraphics[width=\textwidth]{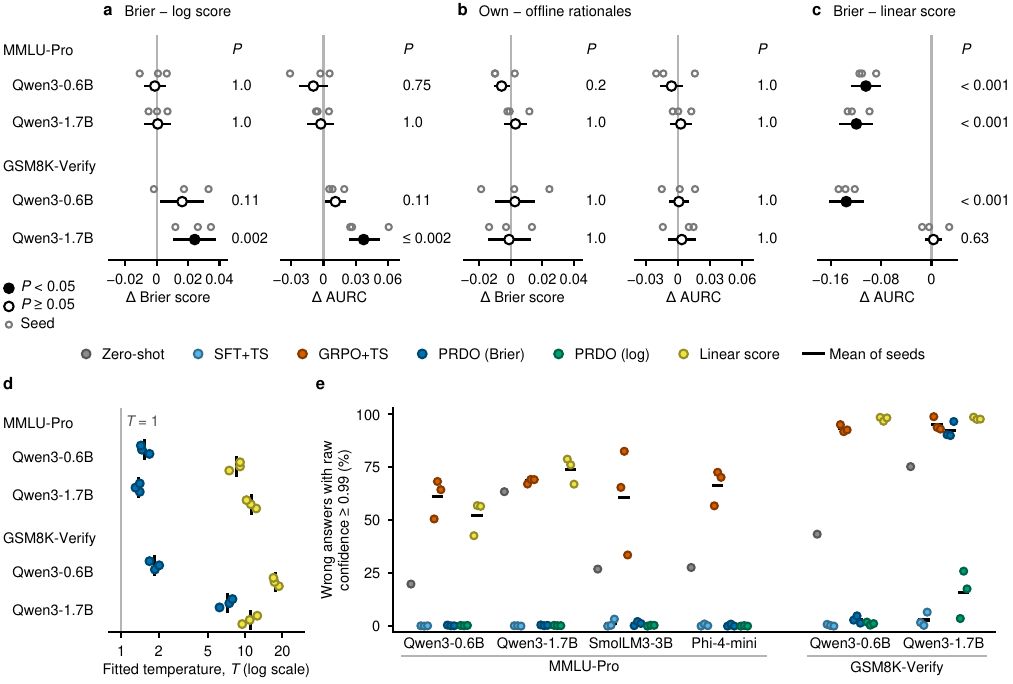}
\caption{\textbf{Proper scoring, rather than a dense gradient alone, prevents saturation of the answer readout.}
\textbf{a--c}, Difference between the Brier objective and the log score (a), frozen-model rationales (b), or the
improper linear score $2u_Y-1$ (c), after temperature scaling. Negative Brier-score and AURC differences favour the
Brier objective. Circles show differences of three-seed means, bars are 95\% item-bootstrap confidence intervals and
grey rings are same-seed differences; filled circles denote Holm-adjusted $P<0.05$. \textbf{d}, Temperature fitted to
the Brier- and linear-score readouts; dots are seeds and bars are means. \textbf{e}, Share of wrong answers assigned
at least 0.99 raw confidence by the zero-shot model, direct-answer SFT, GRPO, and PRDO with
Brier, log or linear scores. Dots are runs and bars are means. MMLU-Pro contains 1,000 test questions and
GSM8K-Verify 1,260; each trained arm has three seeds.}
\label{fig:ingredients}
\end{figure*}

\begin{figure*}[t]
\centering
\includegraphics[width=\textwidth]{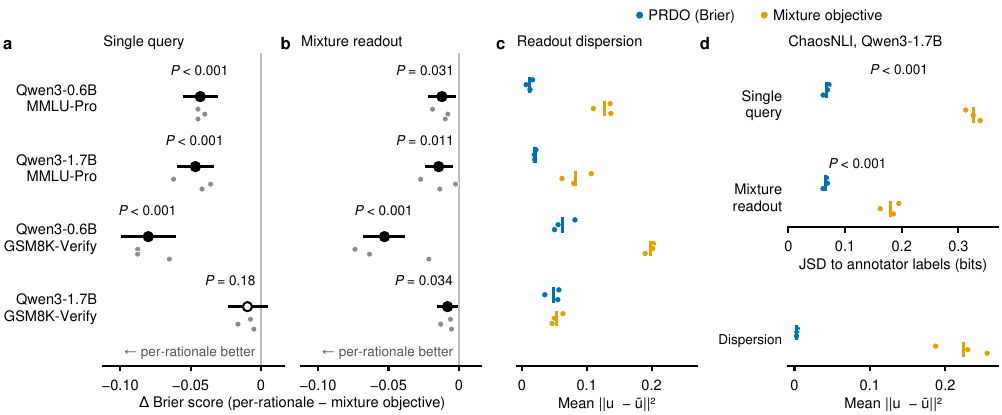}
\caption{\textbf{The supervised unit should match the single-rationale decision.}
\textbf{a}, Difference in single-query Brier score between scoring each rationale's readout and scoring the average
of four readouts. Negative values favour per-rationale scoring. \textbf{b}, The same comparison evaluated on the
four-rationale average. \textbf{c}, Mean dispersion $\lVert u_1-\bar u\rVert^2$ among raw readouts under the two
objectives. \textbf{d}, Jensen--Shannon divergence from the distribution of 100 ChaosNLI annotator labels, for a
single readout and a 16-rationale average; lower is better. In a and b, circles are differences of arm means over
three seeds, bars are 95\% item-bootstrap confidence intervals and grey dots are same-seed differences; filled
circles denote Holm-adjusted $P<0.05$. In c and d, dots are seeds and bars are means. Brier scores are after
temperature scaling. Test sets contain 1,000 MMLU-Pro, 1,260 GSM8K-Verify and 400 ChaosNLI items.}
\label{fig:mixture}
\end{figure*}

\subsection{Better Readouts Yield Usable Coverage and Transfer}

We next evaluated the consequence of the learned ranking rather than another aggregate score. On MMLU-Pro,
PRDO could answer more questions at no more than 20\% realized error than GRPO+TS on all four
backbones (Fig.~\ref{fig:deployment}a,b). The gain in oracle coverage was
\npj{Q3a_vs_grpo.qwen3-0.6b.mp.oracle_cov@0.2.pp} to
\npj{Q3a_vs_grpo.phi-4-mini.mp.oracle_cov@0.2.pp} percentage points, with the same direction in every seed. It also
exceeded direct-answer SFT on three backbones. Oracle coverage selects the threshold with test labels and therefore
diagnoses the ranking rather than a deployable policy.

For an operational check, we selected thresholds on a separate calibration split using a finite-sample risk bound
(Fig.~\ref{fig:deployment}c). PRDO produced a certified 20\%-error threshold in every seed on
Qwen3-1.7B, SmolLM3-3B and Phi-4-mini, and in
\npj{rc.qwen3-0.6b.mp.warm400.0.2.ncert} Qwen3-0.6B seeds. GRPO produced none across its 12 MMLU-Pro runs. Whenever a
threshold was obtained, its realized test error remained below the target. This analysis is descriptive, but it shows
that the ranking advantage can survive threshold selection without access to test labels.

The same ordering generalized beyond the training tasks (Fig.~\ref{fig:deployment}d). MMLU-Pro-trained Qwen models
were evaluated on ARC-Challenge, and GSM8K-Verify-trained models on SVAMP-Verify, using the temperature carried from
the source task. PRDO had lower AURC than GRPO+TS in all four transfer cells and in every seed
(differences \npj{Q3b.qwen3-0.6b.svampv.aurc.diff} to
\npj{Q3b.qwen3-1.7b.svampv.aurc.diff}). The confidence ranking was therefore not only a fit to the source task's
calibration set.

\begin{figure*}[t]
\centering
\includegraphics[width=\textwidth]{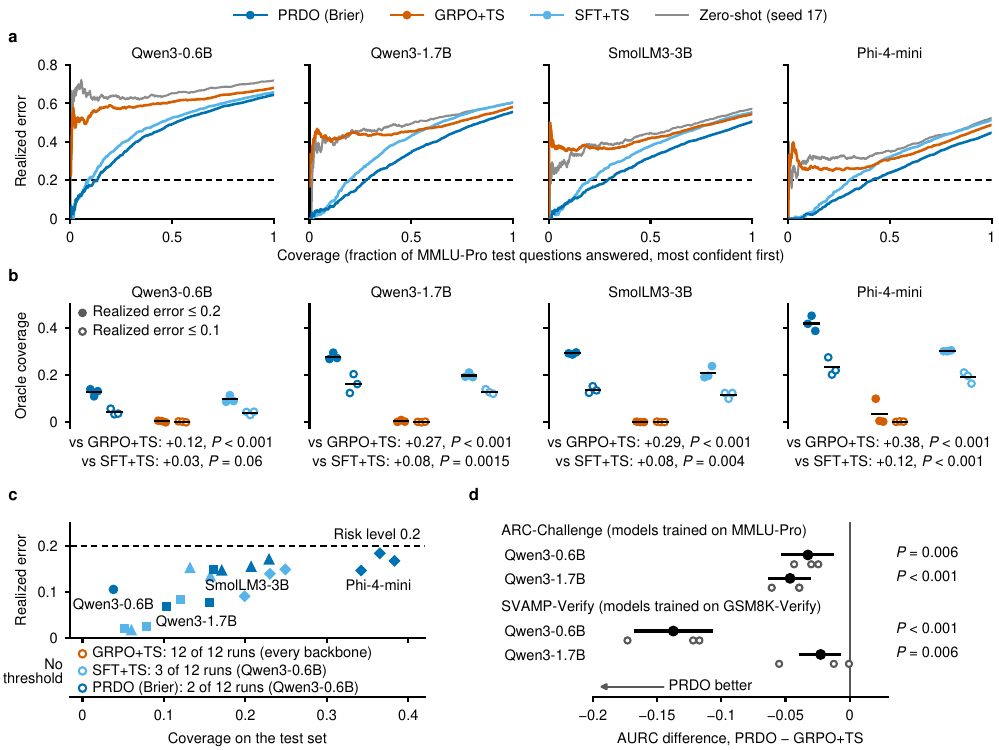}
\caption{\textbf{The improved error ranking supports selective decisions and transfers to unseen tasks.}
\textbf{a}, MMLU-Pro risk--coverage curves: realized error when questions are answered from highest to lowest
confidence. Curves are means over seeds. \textbf{b}, Largest test coverage at no more than 20\% (filled) or 10\% (open)
realized error for PRDO, GRPO+TS and direct-answer SFT+TS. Thresholds use test labels and are an
oracle diagnostic; dots are seeds and bars are means. \textbf{c}, Coverage and realized error when a 20\%-error
threshold is selected on a separate calibration split with a 0.9 finite-sample guarantee. Runs without a qualifying
threshold have zero coverage. \textbf{d}, AURC difference between PRDO and GRPO+TS on unseen
ARC-Challenge and SVAMP-Verify tasks, using the temperature learned on the source task. Dots are three-seed mean
differences, bars are 95\% item-bootstrap confidence intervals and circles are same-seed differences. MMLU-Pro has
1,000, ARC-Challenge 1,172 and SVAMP-Verify 717 test items.}
\label{fig:deployment}
\end{figure*}

\subsection{Decision Quality Is Not Reasoning Quality}

The scope experiments distinguish training the post-rationale distribution from improving the rationale itself.
With longer, capped thinking, the MMLU-Pro advantage over GRPO remained, but PRDO models reached the
2,048-token cap more often, confounding the comparison (Supplementary Fig.~1 and Appendix~\ref{app:thinking}). Adding a
second stage that reinforced rationales with their readout score produced no detectable improvement over the first
stage, and a proper rationale reward did not outperform a sampled-answer reward (Supplementary Fig.~2). The central
method therefore does not require rationale reinforcement, but these experiments do not establish that reasoning
optimization is unnecessary.

Comparisons with RLCR, which rewards a numerical confidence stated after the answer~\citep{damani2025beyond}, and
CREDO, which reads confidence from a dedicated slot after self-analysis~\citep{fan2026credo}, favoured PRDO on
MMLU-Pro and were mixed on GSM8K-Verify (Supplementary Fig.~3). They are scope
comparisons rather than support for the central claim: the CREDO implementation was untuned and the RLCR comparison
was post hoc. Their failure modes nevertheless illustrate the appeal of supervising the decision distribution
directly, without requiring a second output to remain aligned with it.

Most importantly, the verification accuracy trade-off persisted when the Qwen3-0.6B comparison was extended from
three to five seeds. PRDO retained lower AURC on GSM8K-Verify and SVAMP-Verify but lower
accuracy by \npj{PH.X4.five_seeds_gv.qwen3-0.6b.gv.acc.ppabs} and
\npj{PH.X4.five_seeds_svampv.qwen3-0.6b.svampv.acc.ppabs} percentage points, respectively (Appendix~\ref{app:exploratory}). A
model can therefore rank its answers more usefully while solving fewer items at full coverage. This is the boundary
of the method's claim: PRDO improves the distribution used to decide; it is not a general
replacement for objectives that improve the reasoning process.

\section{Discussion and Limitations}

The central result is not simply that one loss outperforms another. It is that reasoning models are often trained on
one object and deployed on another. Outcome rewards optimize whether a sampled trajectory ends correctly; a selective
system acts on the distribution produced after that trajectory. PRDO closes this gap by making the deployed
distribution the training target. Across four backbones, this improved probability quality, error ranking and the
fraction of MMLU-Pro questions that could be accepted at a fixed error rate. Matched controls locate the gain in
proper scoring and per-rationale supervision, rather than in an additional confidence output or merely a denser
gradient.

The distinction between outcome reward and proper scoring is consequential even though deterministic answer keys give
both objectives the same unlimited-capacity optimum. Training never operates at that optimum. Shared representations,
finite capacity and finite updates couple questions that the model cannot yet separate. Over such unresolved cases,
expected correctness rewards concentration on the modal answer, whereas a proper score rewards preserving the outcome
distribution. The linear-score control supports this account: it provided the same direct gradient path as the Brier
score but produced more saturated errors and a worse ranking in three of four cells. This evidence identifies the
incentive created by properness; it does not prove a particular representation-level mechanism.

Per-rationale supervision extends the same alignment principle to the stochastic computation that precedes an answer.
If deployment observes one rationale, it also observes one conditional readout. Scoring only an average can hide poor
constituents through cancellation and, for the Brier score, rewards their disagreement explicitly. The matched
mixture control and the distributed labels of ChaosNLI show that this is not merely a formal concern. The appropriate
training target is the probabilistic object available at the point of action, not an ensemble that the deployed system
does not expose.

The results also clarify what distribution supervision does not do. It need not improve the rationale, and it is not
a universal accuracy objective. Online and frozen-model rationales performed similarly under the anchor, and an
additional rationale-reinforcement stage did not improve the primary outcomes consistently. On mathematical
verification, improved AURC coexisted with lower accuracy, including in the five-seed follow-up. Selective ranking and
full-coverage problem solving should therefore be treated as separate axes. A practical system may combine an
objective that improves reasoning with PRDO for the distribution used for acceptance, but the
interaction between these objectives remains an open optimization problem.

Several limitations bound the evidence. The models contain 0.6--3.8 billion parameters, training is short, answer
spaces are closed and the main comparisons use three seeds. Most rationales are capped at 256 or 320 tokens; the
thinking-mode analysis uses a 2,048-token cap and unequal truncation rates. Forced answer markers mean that some
readouts follow incomplete rationales. GSM8K-Verify uses one small proposal generator and final items drawn from
GSM8K training problems. CREDO was untuned, the RLCR comparison was post hoc, and matched update counts do not imply
identical compute. Whether the same objective remains effective for much larger models, unrestricted answers and
long-horizon reasoning is unknown.

The evaluation consequence is immediate. Calibration error alone is insufficient because a global rescaling can
improve confidence levels without repairing which answers are ranked above which. Accuracy, Brier score, AURC and
coverage at a fixed realized risk answer different questions and should be reported together; operational thresholds
must be selected on data separate from the test set~\citep{geifman2017selective,geifman2019bias}. The training
consequence is equally direct: if a reasoning system decides from a probability distribution after reasoning, that
distribution should not remain an incidental by-product of an outcome reward. It is a first-class object of
optimization. PRDO gives that object a direct training signal while leaving a separate question---how best to improve
the reasoning itself---explicitly open.

\section{Code and Data Availability}

The implementation, training and evaluation configurations, and released result artifacts are available at
\url{https://github.com/ZimmyGao/openjev-rlcd}. The repository records the split rules, analysis procedures and
version history needed to distinguish pre-specified analyses from exploratory follow-ups. The full appendix tables
in this arXiv version report the run-level and comparison-level results used for the claims in the main text.

The public datasets are available from the Hugging Face Hub (MMLU-Pro: TIGER-Lab\slash MMLU-Pro; GSM8K:
openai\slash gsm8k; ARC-Challenge: allenai\slash ai2\_arc; SVAMP: ChilleD\slash SVAMP; ChaosNLI:
tasksource\slash chaos-mnli-ambiguity), as are the models (Qwen\slash Qwen3-0.6B, Qwen\slash Qwen3-1.7B,
HuggingFaceTB\slash SmolLM3-3B and microsoft\slash Phi-4-mini-instruct). Exact model and dataset revisions are
recorded with the released configurations.
\clearpage
\bibliography{references}

@article{brier1950verification,
  author  = {Brier, Glenn W.},
  title   = {Verification of forecasts expressed in terms of probability},
  journal = {Monthly Weather Review},
  volume  = {78},
  number  = {1},
  pages   = {1--3},
  year    = {1950},
}

@article{gneiting2007strictly,
  author  = {Gneiting, Tilmann and Raftery, Adrian E.},
  title   = {Strictly proper scoring rules, prediction, and estimation},
  journal = {Journal of the American Statistical Association},
  volume  = {102},
  number  = {477},
  pages   = {359--378},
  year    = {2007},
}

@inproceedings{guo2017calibration,
  author    = {Guo, Chuan and Pleiss, Geoff and Sun, Yu and Weinberger, Kilian Q.},
  title     = {On calibration of modern neural networks},
  booktitle = {International Conference on Machine Learning (ICML)},
  year      = {2017},
  volume    = {70},
  pages     = {1321--1330},
}

@article{williams1992simple,
  author  = {Williams, Ronald J.},
  title   = {Simple statistical gradient-following algorithms for connectionist reinforcement learning},
  journal = {Machine Learning},
  volume  = {8},
  pages   = {229--256},
  year    = {1992},
}

@inproceedings{kool2019buy,
  author    = {Kool, Wouter and van Hoof, Herke and Welling, Max},
  title     = {Buy 4 {REINFORCE} samples, get a baseline for free!},
  booktitle = {ICLR Workshop on Deep Reinforcement Learning Meets Structured Prediction},
  year      = {2019},
}

@inproceedings{ahmadian2024back,
  author    = {Ahmadian, Arash and Cremer, Chris and Gall{\'e}, Matthias and Fadaee, Marzieh and Kreutzer, Julia and Pietquin, Olivier and {\"U}st{\"u}n, Ahmet and Hooker, Sara},
  title     = {Back to basics: Revisiting {REINFORCE}-style optimization for learning from human feedback in {LLMs}},
  booktitle = {Annual Meeting of the Association for Computational Linguistics (ACL)},
  year      = {2024},
  pages     = {12248--12267},
}

@misc{shao2024deepseekmath,
  author  = {Shao, Zhihong and Wang, Peiyi and Zhu, Qihao and Xu, Runxin and Song, Junxiao and Zhang, Mingchuan and Li, Y. K. and Wu, Y. and Guo, Daya},
  title   = {{DeepSeekMath}: Pushing the limits of mathematical reasoning in open language models},
  howpublished = {Preprint at \url{https://arxiv.org/abs/2402.03300}},
  year    = {2024},
}

@inproceedings{zelikman2022star,
  author    = {Zelikman, Eric and Wu, Yuhuai and Mu, Jesse and Goodman, Noah D.},
  title     = {{STaR}: Bootstrapping reasoning with reasoning},
  booktitle = {Advances in Neural Information Processing Systems (NeurIPS)},
  year      = {2022},
  volume    = {35},
  pages     = {15476--15488},
}

@inproceedings{damani2025beyond,
  author    = {Damani, Mehul and Puri, Isha and Slocum, Stewart and Shenfeld, Idan and Choshen, Leshem and Kim, Yoon and Andreas, Jacob},
  title     = {Beyond binary rewards: Training {LMs} to reason about their uncertainty},
  booktitle = {International Conference on Learning Representations (ICLR)},
  year      = {2026},
}

@inproceedings{geifman2017selective,
  author    = {Geifman, Yonatan and El-Yaniv, Ran},
  title     = {Selective classification for deep neural networks},
  booktitle = {Advances in Neural Information Processing Systems (NeurIPS)},
  year      = {2017},
  volume    = {30},
  pages     = {4878--4887},
}

@inproceedings{wang2024mmlupro,
  author    = {Wang, Yubo and Ma, Xueguang and Zhang, Ge and Ni, Yuansheng and Chandra, Abhranil and Guo, Shiguang and Ren, Weiming and Arulraj, Aaran and He, Xuan and Jiang, Ziyan and Li, Tianle and Ku, Max and Wang, Kai and Zhuang, Alex and Fan, Rongqi and Yue, Xiang and Chen, Wenhu},
  title     = {{MMLU-Pro}: A more robust and challenging multi-task language understanding benchmark},
  booktitle = {Advances in Neural Information Processing Systems (NeurIPS) Datasets and Benchmarks Track},
  year      = {2024},
  volume    = {37},
  pages     = {95266--95290},
}

@misc{cobbe2021training,
  author  = {Cobbe, Karl and Kosaraju, Vineet and Bavarian, Mohammad and Chen, Mark and Jun, Heewoo and Kaiser, Lukasz and Plappert, Matthias and Tworek, Jerry and Hilton, Jacob and Nakano, Reiichiro and Hesse, Christopher and Schulman, John},
  title   = {Training verifiers to solve math word problems},
  howpublished = {Preprint at \url{https://arxiv.org/abs/2110.14168}},
  year    = {2021},
}

@inproceedings{nie2020learn,
  author    = {Nie, Yixin and Zhou, Xiang and Bansal, Mohit},
  title     = {What can we learn from collective human opinions on natural language inference data?},
  booktitle = {Conference on Empirical Methods in Natural Language Processing (EMNLP)},
  year      = {2020},
  pages     = {9131--9143},
}

@misc{qwen3,
  author  = {{Qwen Team}},
  title   = {{Qwen3} technical report},
  howpublished = {Preprint at \url{https://arxiv.org/abs/2505.09388}},
  year    = {2025},
}

@misc{phi4mini,
  author  = {{Microsoft}},
  title   = {{Phi-4-Mini} technical report: Compact yet powerful multimodal language models via mixture-of-{LoRAs}},
  howpublished = {Preprint at \url{https://arxiv.org/abs/2503.01743}},
  year    = {2025},
}

@misc{smollm3,
  author       = {{Hugging Face}},
  title        = {{SmolLM3}: smol, multilingual, long-context reasoner},
  howpublished = {\url{https://huggingface.co/blog/smollm3}},
  year         = {2025},
}

@inproceedings{ma2026decoupling,
  author    = {Ma, Zhengzhao and Wen, Xueru and Cao, Boxi and Lu, Yaojie and Lin, Hongyu and Yang, Jinglin and He, Min and Han, Xianpei and Sun, Le},
  title     = {Decoupling reasoning and confidence: Resurrecting calibration in reinforcement learning from verifiable rewards},
  booktitle = {International Conference on Machine Learning (ICML)},
  year      = {2026},
}

@inproceedings{wang2026capo,
  author    = {Wang, Ziqi and Lou, Xingzhou and Wu, Meiqi and Wen, Zhengqi and Zhang, Junge},
  title     = {Calibration-aware policy optimization for reasoning {LLMs}},
  booktitle = {Annual Meeting of the Association for Computational Linguistics (ACL)},
  pages     = {18375--18390},
  year      = {2026},
}

@misc{fan2026credo,
  author  = {Fan, Chenxiao and Gao, Chongming and Zhang, Gangyi and Shen, Leyang and Gong, Yaxin and Wang, Jiamin and Wang, Jiakai and Wang, Dong and Liu, Yang and Feng, Fuli and He, Xiangnan},
  title   = {Beyond verbalized confidence: Calibrating reasoners with differentiable readouts},
  howpublished = {Preprint at \url{https://arxiv.org/abs/2609.34857}},
  year    = {2026},
}

@inproceedings{baniharouni2026rewarding,
  author    = {Bani-Harouni, David and Pellegrini, Chantal and Stangel, Paul and {\"O}zsoy, Ege and Zaripova, Kamilia and Keicher, Matthias and Navab, Nassir},
  title     = {Rewarding doubt: A reinforcement learning approach to calibrated confidence expression of large language models},
  booktitle = {International Conference on Learning Representations (ICLR)},
  year      = {2026},
}

@inproceedings{krogh1995neural,
  author    = {Krogh, Anders and Vedelsby, Jesper},
  title     = {Neural network ensembles, cross validation, and active learning},
  booktitle = {Advances in Neural Information Processing Systems 7},
  pages     = {231--238},
  year      = {1995},
}

@inproceedings{williams2018broad,
  author    = {Williams, Adina and Nangia, Nikita and Bowman, Samuel R.},
  title     = {A broad-coverage challenge corpus for sentence understanding through inference},
  booktitle = {Conference of the North American Chapter of the Association for Computational Linguistics (NAACL)},
  pages     = {1112--1122},
  year      = {2018},
}

@inproceedings{loshchilov2019decoupled,
  author    = {Loshchilov, Ilya and Hutter, Frank},
  title     = {Decoupled weight decay regularization},
  booktitle = {International Conference on Learning Representations (ICLR)},
  year      = {2019},
}

@inproceedings{naeini2015obtaining,
  author    = {Naeini, Mahdi Pakdaman and Cooper, Gregory F. and Hauskrecht, Milos},
  title     = {Obtaining well calibrated probabilities using {B}ayesian binning},
  booktitle = {AAAI Conference on Artificial Intelligence},
  pages     = {2901--2907},
  year      = {2015},
}

@inproceedings{geifman2019bias,
  author    = {Geifman, Yonatan and Uziel, Guy and El-Yaniv, Ran},
  title     = {Bias-reduced uncertainty estimation for deep neural classifiers},
  booktitle = {International Conference on Learning Representations (ICLR)},
  year      = {2019},
}

@article{holm1979simple,
  author  = {Holm, Sture},
  title   = {A simple sequentially rejective multiple test procedure},
  journal = {Scandinavian Journal of Statistics},
  volume  = {6},
  number  = {2},
  pages   = {65--70},
  year    = {1979},
}

@misc{clark2018arc,
  author  = {Clark, Peter and Cowhey, Isaac and Etzioni, Oren and Khot, Tushar and Sabharwal, Ashish and Schoenick, Carissa and Tafjord, Oyvind},
  title   = {Think you have solved question answering? {Try} {ARC}, the {AI2} reasoning challenge},
  howpublished = {Preprint at \url{https://arxiv.org/abs/1803.05457}},
  year    = {2018},
}

@inproceedings{patel2021svamp,
  author    = {Patel, Arkil and Bhattamishra, Satwik and Goyal, Navin},
  title     = {Are {NLP} models really able to solve simple math word problems?},
  booktitle = {Proceedings of the 2021 Conference of the North American Chapter of the Association for Computational Linguistics: Human Language Technologies (NAACL-HLT)},
  pages     = {2080--2094},
  year      = {2021},
}

@misc{wang2026unijev,
  author  = {Wang, Jianing and Jian, Youwen and Wang, Derong and Yu, Jinyang and Si, Qingyi and Qin, Chuanyu and Xiao, Yanghua and Guo, Hongcheng},
  title   = {{UniJev}: Separating accuracy from calibration in trimodal {System One} decision models},
  howpublished = {Preprint at \url{https://arxiv.org/abs/2610.25784}},
  year    = {2026},
}

@article{clopper1934use,
  author  = {Clopper, C. J. and Pearson, E. S.},
  title   = {The use of confidence or fiducial limits illustrated in the case of the binomial},
  journal = {Biometrika},
  volume  = {26},
  number  = {4},
  pages   = {404--413},
  year    = {1934},
}

@article{murphy1973new,
  author  = {Murphy, Allan H.},
  title   = {A new vector partition of the probability score},
  journal = {Journal of Applied Meteorology},
  volume  = {12},
  number  = {4},
  pages   = {595--600},
  year    = {1973},
}

@article{lin1991divergence,
  author  = {Lin, Jianhua},
  title   = {Divergence measures based on the {Shannon} entropy},
  journal = {IEEE Transactions on Information Theory},
  volume  = {37},
  number  = {1},
  pages   = {145--151},
  year    = {1991},
}

@book{efron1993introduction,
  author    = {Efron, Bradley and Tibshirani, Robert J.},
  title     = {An Introduction to the Bootstrap},
  publisher = {Chapman \& Hall},
  address   = {New York},
  year      = {1993},
}

@inproceedings{paszke2019pytorch,
  author    = {Paszke, Adam and Gross, Sam and Massa, Francisco and Lerer, Adam and Bradbury, James and Chanan, Gregory and Killeen, Trevor and Lin, Zeming and Gimelshein, Natalia and Antiga, Luca and others},
  title     = {{PyTorch}: an imperative style, high-performance deep learning library},
  booktitle = {Advances in Neural Information Processing Systems},
  volume    = {32},
  year      = {2019},
  pages     = {8024--8035},
}

@inproceedings{wolf2020transformers,
  author    = {Wolf, Thomas and Debut, Lysandre and Sanh, Victor and Chaumond, Julien and Delangue, Clement and Moi, Anthony and Cistac, Pierric and Rault, Tim and Louf, R{\'e}mi and Funtowicz, Morgan and others},
  title     = {Transformers: state-of-the-art natural language processing},
  booktitle = {Proceedings of the 2020 Conference on Empirical Methods in Natural Language Processing: System Demonstrations},
  pages     = {38--45},
  year      = {2020},
}

@misc{angelopoulos2021learn,
  author  = {Angelopoulos, Anastasios N. and Bates, Stephen and Cand{\`e}s, Emmanuel J. and Jordan, Michael I. and Lei, Lihua},
  title   = {Learn then test: calibrating predictive algorithms to achieve risk control},
  howpublished = {Preprint at \url{https://arxiv.org/abs/2110.01052}},
  year    = {2021},
}
\bibliographystyle{icml2026}

\clearpage
\appendix
\onecolumn
\setcounter{figure}{0}
\setcounter{table}{0}
\renewcommand{\thefigure}{\arabic{figure}}
\renewcommand{\thetable}{\arabic{table}}
\renewcommand{\figurename}{Supplementary Figure}
\renewcommand{\tablename}{Supplementary Table}

\noindent Throughout, PRDO denotes Post-Rationale Distribution Optimization.

\section{Evaluation and Analysis Provenance}\label{app:provenance}

The main study used a prospectively specified analysis plan rather than an externally registered protocol. It
fixed the main splits, training arms, baseline-selection rules, endpoints, statistical tests and cross-model
criteria. The plan and its revision history are archived with the project. The main plan was fixed before any
full final-test run began.
Code smoke tests, using models trained for two to four steps, had already evaluated the first four final-test and
calibration items on Qwen3-0.6B. Further smoke tests evaluated the same four MMLU-Pro final-test items on SmolLM3-3B
and Phi-4-mini, in most runs also the same four calibration items. These limited exposures are disclosed rather than
treating the final sets as entirely unread.

The transfer, ChaosNLI and thinking-mode extensions were specified before their full experiments began. Some extension records and the statistical-analysis code were finalized after the main experiments had been
scheduled but before any main final-test run finished. The thinking-mode revision was finalized before those runs
began; it restricted training to one cell and added evaluation-only comparisons for computational feasibility. CREDO settings and comparisons were fixed before its runs. The RLCR comparison is post hoc because
completed-run logs containing final-test metrics had been inspected before that comparison was specified.

After the main comparisons were computed, exploratory experiments examined the rationale reward, budget-matched
continuations, anchor removal and two additional verification seeds (Appendix~\ref{app:exploratory}; the archived exploratory analysis plan).
Their analysis code preceded their results; they do not alter the original decision rules. The main decision code
had omitted the planned secondary coverage endpoints, which were subsequently computed using the same test
functions and are labelled post hoc in Supplementary Table~1. Split hashes, code-version records and the full
chronology allow these distinctions to be checked independently.

\section{Statistical Tests, Monte Carlo Floors and Robustness Checks}\label{app:statistics}

\paragraph{Per-item metrics.} These are accuracy, Brier score and negative log-likelihood (NLL); on ChaosNLI, the
Jensen--Shannon divergence (JSD) and the squared distance to the annotator distribution.
\begin{itemize}
\item Each arm's per-item values are averaged over its seeds, and the arms are compared on the same test items.
\item The 95\% confidence interval (CI) is a percentile bootstrap over items (20,000 resamples).
\item The $P$ value is a two-sided sign-flip permutation test (20,000 flips).
\end{itemize}

\paragraph{Ranking metrics.} These are AURC, E-AURC and oracle coverage.
\begin{itemize}
\item They are computed per seed on resampled items and averaged over seeds, in an item bootstrap with 4,000
resamples.
\item The two-sided $P$ is twice the smaller tail share.
\end{itemize}

\paragraph{Monte Carlo floors.} The smallest attainable unadjusted $P$ is $1/20{,}001$ (sign flips) or $1/4{,}001$
(bootstrap).
\begin{itemize}
\item Holm's adjustment within a family of $m$ tests then starts from the floor, so the adjusted value reported for
such a comparison (at most $m$ times the floor) is an upper bound.
\item The figures and tables mark such values with $\leq$; values below 0.001 are written as $P<0.001$.
\item This is why every MMLU-Pro AURC comparison of the central family reads $P\leq0.002$, however far its interval
lies from zero.
\end{itemize}

\paragraph{Robustness.}
\begin{itemize}
\item \emph{Seed agreement} counts the seeds whose own difference has the sign of the mean difference.
\item The \emph{seed-and-item hierarchical bootstrap} resamples seeds, then items, to give a second interval.
\end{itemize}
Where that interval includes zero although the paired test is significant, the main text or a figure legend says so.
Supplementary Table~1 lists the seed agreement for every comparison, and the hierarchical interval for accuracy, Brier
score, NLL, AURC and the ChaosNLI metrics except in the thinking-versus-non-thinking comparisons (thinking-mode analysis)
and three post-hoc ones: the stage-1 continuation versus stage 1 (X2) and each arm trained without the token anchor
versus its anchored counterpart (X3).

The following identifiers map the scientific comparisons to the unchanged machine-readable results and archived
analysis plan; they are lookup keys, not additional scientific concepts.

\paragraph{Complete comparison families and evidence criteria.} $P$ values are Holm-adjusted within the comparison families described in the main-text Experimental Setup
unless stated; an arm is better or worse at an adjusted $P<0.05$. C (six cells, four MMLU-Pro and two GSM8K-Verify,
$\times$ \{Brier, AURC\}): supported if PRDO beats GRPO+TS in both metrics on at least three of
the four MMLU-Pro backbones and loses in neither on any, refuted if it wins in neither on at least three, else
inconclusive; GSM8K-Verify is reported, not decisive. Q1a (Brier versus log score) and Q1b (own versus offline
rationales), each four Qwen3 cells $\times$ \{Brier, AURC\}: a score is claimed better only if it wins one metric in at
least three cells; own rationales are supported if they win in Brier or AURC in at least three cells and lose in none,
refuted if offline ones win in at least two. Q1c (Brier versus linear score, four cells $\times$ AURC): supported if
Brier wins in at least three cells and the linear score's mean fitted temperature (over kept seeds, compared without a test) is higher in every cell. Q2
(per-rationale versus mixture objective, four cells $\times$ single-query Brier): supported if the per-rationale
objective wins in at least three cells (P2a) and the mixture objective's readouts disperse more in every cell (P2b; mean dispersion over kept seeds, compared without a test),
refuted if the mixture objective wins in at least two; P2c (its own family) predicted that the mixture objective would
not lose on the mixture readout. Q3a (per comparator, GRPO+TS or SFT+TS; four MMLU-Pro backbones $\times$ oracle
coverage at error 0.2) and Q3b (four transfer cells $\times$ AURC): supported if better in at least three cells. A
proper second-stage reward is claimed preferable only if the two-stage arm beats the sampled-answer reward in Brier or
AURC on both tasks (two Qwen3-1.7B cells $\times$ \{Brier, AURC\}). Reported only: stage 2 versus stage 1 (four Qwen3
cells $\times$ \{accuracy, Brier, AURC\}) and each closest method (four Qwen3 cells $\times$ \{Brier, AURC\}). Separate
families: accuracy, NLL and E-AURC beside C, Q1a--c, Q2, Q3b, the stage-2 reward and each closest method, and, post
hoc, beside Q3a; coverage at error 0.1 (Q3a); AURC with the target-task temperature (Q3b); the ChaosNLI tests (unscaled and scaled readouts, each over \{single query, mixture\} $\times$ \{JSD, squared distance\}; majority accuracy over the two readouts); and the descriptive thinking-mode analyses: PRDO versus GRPO+TS in thinking mode (cells $\times$ \{Brier, AURC\}, accuracy separately) and, per arm, thinking-mode versus non-thinking results on the same items (cells $\times$ \{Brier, AURC, accuracy\}; in design E the same weights, in design B the arm trained with and without thinking), over the four cells of design E and the one cell of design B.

\section{Thinking Mode and Truncation}\label{app:thinking}

\paragraph{Cap.} In the pre-specified length probe of the zero-shot models (300 development questions, cap 2,048 tokens),
run before the full thinking-mode experiments and recorded in the evaluation plan, the rule (90\% of rationales finished) was met in no cell, so the cap was 2,048 tokens everywhere; the probe results are
in the archived analysis plan.

\paragraph{Design.} Two parts, each read out as a single query only:
\begin{itemize}
\item training in thinking mode, on Qwen3-1.7B GSM8K-Verify only;
\item evaluation in thinking mode of the weights trained without thinking, on the four Qwen3 cells.
\end{itemize}

\paragraph{Truncation.} A rationale that reaches the cap gets the answer marker appended, so its readout follows
truncated reasoning.
\begin{itemize}
\item PRDO reached the cap more often than GRPO+TS: on Qwen3-0.6B MMLU-Pro,
\npj{arm.qwen3-0.6b.mp.thinkeval_warm400.length_limit_rate}\% of test rationales against
\npj{arm.qwen3-0.6b.mp.thinkeval_grpo.length_limit_rate}\%.
\item On Qwen3-0.6B GSM8K-Verify the figures were
\npj{arm.qwen3-0.6b.gv.thinkeval_warm400.length_limit_rate}\% against
\npj{arm.qwen3-0.6b.gv.thinkeval_grpo.length_limit_rate}\%.
\end{itemize}
The thinking-mode comparison (Supplementary Fig.~1) therefore mixes the quality of the readout with the rate of truncation. It is
reported descriptively, as specified in the evaluation plan.

\section{CREDO+TS Reproduction}\label{app:credo}

\paragraph{Settings.} CREDO+TS follows the published method with settings specified before its runs, which were fixed
before any CREDO run, and the adaptations listed in Section~\ref{sec:experimental-setup}.
\begin{itemize}
\item It was not tuned.
\item Its analysis cap of 128 tokens was set by a probe on development questions before any training job.
\item Supplementary Table~5 gives every run, including the forced-marker rate reported per
run, and the top-label Brier score with its Murphy terms.
\end{itemize}

\paragraph{Readout dynamics.} The slot readout saturated during training (Supplementary Fig.~3c).
\begin{itemize}
\item On GSM8K-Verify the mean slot confidence of the training batch went to 1 within a few dozen steps. The
expected accuracy of the sampled answers there exceeded 0.5 at every logged step.
\item For Qwen3-0.6B, seed 17, it was \npj{trainlog.qwen3-0.6b.gv.credo.17.step20.conf} at step 20 and
\npj{trainlog.qwen3-0.6b.gv.credo.17.step40.conf} at step 40. The expected accuracy of the sampled answers was
\npj{trainlog.qwen3-0.6b.gv.credo.17.step40.acc} at step 40.
\item On the test set, every single-query readout of four of the six GSM8K-Verify runs had a logit gap above 10,
i.e. a confidence above 0.99995.
\item On MMLU-Pro the batch mean stayed near 0 for most logged steps in three runs and swung widely in the other
three.
\end{itemize}

\paragraph{The excluded run.} The run on Qwen3-1.7B MMLU-Pro with seed 23 collapsed to rationales of
\npj{run.qwen3-1.7b.mp.credo.23.tokens} tokens on average. The pre-specified collapse rule therefore excluded it.

\paragraph{Implementation check.} A read-only review compared our implementation with the published description. It
covered:
\begin{itemize}
\item numerical checks of the loss gradients;
\item the position and tokens of the slot readout;
\item the alignment of answers, analyses and readouts across rollouts.
\end{itemize}
It found no defect.

\paragraph{A likely explanation, not tested here.}
\begin{itemize}
\item The base model's slot readout was already nearly binary.
\item With such a readout, a group-standardized confidence reward favours the majority outcome of each group.
\item The calibration term has almost no gradient once the readout saturates.
\end{itemize}
Because the method ran untuned under our budget and adaptations, these results describe CREDO+TS in this setting, not
its best attainable performance.

\section{Post-Hoc Exploratory Experiments}\label{app:exploratory}

The pre-specified results left three points open:
\begin{itemize}
\item training on the model's own rationales was not distinguishable from training on offline rationales;
\item the advantage of a proper reward for the reasoning was not established;
\item on Qwen3-0.6B GSM8K-Verify, PRDO was less accurate than GRPO+TS.
\end{itemize}
After the decisions had been computed, we designed three exploratory experiments, X1--X3, recorded in the exploratory analysis plan, and later a fourth, X4, as a seed-extension experiment (the archived exploratory analysis plan gives the design, the analysis and, for
X1--X3, the prediction of each, fixed before its runs). They change no pre-specified decision. Their comparisons are in the last block of Supplementary Table~1, each family with its own Holm
adjustment.

\paragraph{X1: the reward for the reasoning, on ChaosNLI.} With one correct answer per question, the proper and the
sampled-answer reward share their optimum. ChaosNLI gives each item a distribution of annotator labels, and there the
optima differ: the expected Brier reward peaks at the annotator distribution, the expected sampled-answer reward at its
most frequent label.
\begin{itemize}
\item \emph{Design.} From the per-rationale ChaosNLI stage-1 checkpoints (Qwen3-1.7B, seeds 17, 23 and 29),
400 further steps reinforced the reasoning with the proper (Brier) reward or with the sampled-answer reward; the
readout kept its proper pathwise term in both. A third arm continued stage 1 for the same 400 steps (budget control).
\item \emph{Prediction.} The proper reward gives a lower JSD and squared distance to the annotator distribution
(unscaled readouts primary, as in the main ChaosNLI comparison; Holm over two readouts $\times$ two metrics).
\item \emph{Result.} The prediction was not borne out. Unscaled single-query JSD, mean over seeds: proper reward
\npj{PH.X1.mean.cn_s2_rl.single.raw.jsd}, sampled-answer reward \npj{PH.X1.mean.cn_s2_correct.single.raw.jsd},
continuation \npj{PH.X1.mean.cn_s2_cont.single.raw.jsd}, stage 1 \npj{PH.X1.mean.cn_s1_lam0.single.raw.jsd}.
\item The paired test favoured the sampled-answer reward (proper minus sampled-answer, single-query JSD
\npj{PH.X1.proper_vs_sampled_raw.single.jsd.diff}, 95\% CI \npj{PH.X1.proper_vs_sampled_raw.single.jsd.ci},
Holm-adjusted \npj{PH.X1.proper_vs_sampled_raw.single.jsd.p}); all four tests of the family did so, each at Holm-adjusted
\npj{PH.X1.proper_vs_sampled_raw.single.jsd.p} or smaller. None was robust: in each,
\npj{PH.X1.proper_vs_sampled_raw.single.jsd.seeds} seeds agreed in sign and the seed-and-item interval included zero
(single-query JSD, \npj{PH.X1.proper_vs_sampled_raw.single.jsd.hier}); seed 23 contributed most of the single-query
JSD difference (\npj{run.qwen3-1.7b.chaosnli.cn_s2_rl.23.single.raw.jsd} against
\npj{run.qwen3-1.7b.chaosnli.cn_s2_correct.23.single.raw.jsd}).
\item After temperature scaling, three of the four tests showed no difference (Holm-adjusted
\npj{PH.X1.proper_vs_sampled_ts.single.sqdist.p} or larger); the fourth (mixture JSD,
\npj{PH.X1.proper_vs_sampled_ts.mixture.jsd.p}) again favoured the sampled-answer reward, but only
\npj{PH.X1.proper_vs_sampled_ts.mixture.jsd.seeds} seeds agreed in sign.
\item Against the continuation, whose single-query JSD stayed at the level of stage 1, the proper reward lowered it
(\npj{PH.X1.proper_vs_continuation_raw.single.jsd.diff}, Holm-adjusted
\npj{PH.X1.proper_vs_continuation_raw.single.jsd.p}, \npj{PH.X1.proper_vs_continuation_raw.single.jsd.seeds} seeds),
but the seed-and-item interval (\npj{PH.X1.proper_vs_continuation_raw.single.jsd.hier}) included zero; the squared
distance did not differ (\npj{PH.X1.proper_vs_continuation_raw.single.sqdist.p}).
\item \emph{A likely reason, not tested here.} The readouts after different rationales for the same item hardly
differed (mean dispersion \npj{arm.qwen3-1.7b.chaosnli.cn_s1_lam0.dispersion} in the stage-1 checkpoints the arms
started from and
\npj{arm.qwen3-1.7b.chaosnli.cn_s2_rl.dispersion} after the proper reward, against
\npj{arm.qwen3-1.7b.chaosnli.cn_s1_lam1.dispersion} under the mixture objective), so a reward that scores each
rationale's readout could separate rationales only weakly. In this setting, whether the reward for the reasoning is
proper made no robust difference; the proper score matters where the answer distribution itself is trained.
\end{itemize}

\paragraph{X2: budget-matched continuations.} The two-stage model trained for 800 steps, stage~1 for 400, so their
pre-specified comparison (Results) mixed the reward for the reasoning with extra training. And on Qwen3-0.6B
GSM8K-Verify, stage~1, which trains the readout and only anchors the reasoning, was less accurate than GRPO+TS, which
reinforces the reasoning.
\begin{itemize}
\item \emph{Design.} From each stage-1 (Brier) checkpoint, 400 further steps with seed $1000+s$, as in the registered
second stage: a stage-1 continuation, that is, the second stage without its score-function term (the budget control;
four Qwen3 cells), and a GRPO continuation with the registered recipe (Qwen3-0.6B, both tasks; Qwen3-1.7B has the
registered one). Every continuation thus trained for 800 steps in all, like the two-stage model; GRPO+TS trained for
400.
\item \emph{Prediction.} On GSM8K-Verify, the continuations that reinforce the reasoning are more accurate than the
stage-1 continuation. Families: the two-stage model minus the stage-1 continuation in accuracy, Brier score and AURC
(Holm over four cells $\times$ three metrics); on Qwen3-0.6B, the GRPO continuation minus the stage-1 continuation and
minus GRPO+TS (each Holm over two cells $\times$ three metrics). The analysis script, written before the results were
read, also compares the stage-1 continuation with stage~1 (Holm over four cells $\times$ three metrics; no
seed-and-item interval), a comparison not listed in the archived exploratory analysis plan.
\item \emph{Result.} The prediction held on Qwen3-0.6B, robustly only for the GRPO continuation, but not on
Qwen3-1.7B, where stage~1 had not differed in accuracy from GRPO+TS: there the two-stage model was not more accurate
than the stage-1 continuation (\npj{PH.X2.twostage_vs_s1cont.qwen3-1.7b.gv.acc.pp} percentage points, Holm-adjusted
\npj{PH.X2.twostage_vs_s1cont.qwen3-1.7b.gv.acc.p}), and the registered GRPO continuation, not tested, was about as accurate
(\npj{arm.qwen3-1.7b.gv.fork_grpo.single.ts.acc.pct}\% against \npj{arm.qwen3-1.7b.gv.s1cont.single.ts.acc.pct}\%).
Accuracy on Qwen3-0.6B GSM8K-Verify, mean over seeds (rounded; the differences below were computed before rounding):
stage-1 continuation
\npj{arm.qwen3-0.6b.gv.s1cont.single.ts.acc.pct}\%, two-stage \npj{arm.qwen3-0.6b.gv.fork_rl.single.ts.acc.pct}\%,
GRPO continuation \npj{arm.qwen3-0.6b.gv.fork_grpo.single.ts.acc.pct}\% (stage~1,
\npj{arm.qwen3-0.6b.gv.warm400.single.ts.acc.pct}\%; GRPO+TS, \npj{arm.qwen3-0.6b.gv.grpo.single.ts.acc.pct}\%).
\item \emph{Two-stage minus stage-1 continuation.} On Qwen3-0.6B GSM8K-Verify, accuracy
\npj{PH.X2.twostage_vs_s1cont.qwen3-0.6b.gv.acc.pp} percentage points (95\% CI
\npj{PH.X2.twostage_vs_s1cont.qwen3-0.6b.gv.acc.cipp}, Holm-adjusted \npj{PH.X2.twostage_vs_s1cont.qwen3-0.6b.gv.acc.p},
\npj{PH.X2.twostage_vs_s1cont.qwen3-0.6b.gv.acc.seeds} seeds) and Brier score
\npj{PH.X2.twostage_vs_s1cont.qwen3-0.6b.gv.brier.diff} (\npj{PH.X2.twostage_vs_s1cont.qwen3-0.6b.gv.brier.p},
\npj{PH.X2.twostage_vs_s1cont.qwen3-0.6b.gv.brier.seeds} seeds), neither robust: both seed-and-item intervals included
zero (accuracy, \npj{PH.X2.twostage_vs_s1cont.qwen3-0.6b.gv.acc.hierpp} points; Brier score,
\npj{PH.X2.twostage_vs_s1cont.qwen3-0.6b.gv.brier.hier}). The accuracy gain was larger than the registered one
over stage~1 (\npj{stage2_vs_stage1.qwen3-0.6b.gv.acc.pp} points, Holm-adjusted
\npj{stage2_vs_stage1.qwen3-0.6b.gv.acc.p}) only because the stage-1 continuation was slightly, and not significantly,
less accurate than stage~1 (below). The AURC did not
differ after adjustment (\npj{PH.X2.twostage_vs_s1cont.qwen3-0.6b.gv.aurc.diff},
\npj{PH.X2.twostage_vs_s1cont.qwen3-0.6b.gv.aurc.p}), and no test differed in the other three cells (Holm-adjusted
\npj{PH.X2.twostage_vs_s1cont.qwen3-1.7b.mp.brier.p} or larger). In the pre-specified comparison with stage~1, all twelve
point estimates had favoured the two-stage model; against the stage-1 continuation,
\npj{PH.X2.twostage_vs_s1cont.nfavour} of \npj{PH.X2.twostage_vs_s1cont.ntests} did.
\item \emph{GRPO continuation minus stage-1 continuation (Qwen3-0.6B).} On GSM8K-Verify, better in all three metrics,
in every seed and with every seed-and-item interval excluding zero: accuracy
\npj{PH.X2.grpocont_vs_s1cont_06b.qwen3-0.6b.gv.acc.pp} percentage points (95\% CI
\npj{PH.X2.grpocont_vs_s1cont_06b.qwen3-0.6b.gv.acc.cipp}), Brier score
\npj{PH.X2.grpocont_vs_s1cont_06b.qwen3-0.6b.gv.brier.diff}, AURC
\npj{PH.X2.grpocont_vs_s1cont_06b.qwen3-0.6b.gv.aurc.diff} (each Holm-adjusted
\npj{PH.X2.grpocont_vs_s1cont_06b.qwen3-0.6b.gv.aurc.p} or smaller), although its raw readout had moved towards
certainty (\npj{arm.qwen3-0.6b.gv.fork_grpo.satwrong}\% of wrong answers at a raw confidence of at least 0.99, against
\npj{arm.qwen3-0.6b.gv.s1cont.satwrong}\%). On MMLU-Pro, accuracy did not differ
(\npj{PH.X2.grpocont_vs_s1cont_06b.qwen3-0.6b.mp.acc.pp} points,
\npj{PH.X2.grpocont_vs_s1cont_06b.qwen3-0.6b.mp.acc.p}), but the Brier score and AURC were worse
(\npj{PH.X2.grpocont_vs_s1cont_06b.qwen3-0.6b.mp.brier.diff}, \npj{PH.X2.grpocont_vs_s1cont_06b.qwen3-0.6b.mp.brier.p};
\npj{PH.X2.grpocont_vs_s1cont_06b.qwen3-0.6b.mp.aurc.diff}, \npj{PH.X2.grpocont_vs_s1cont_06b.qwen3-0.6b.mp.aurc.p};
every seed agreed; seed-and-item intervals \npj{PH.X2.grpocont_vs_s1cont_06b.qwen3-0.6b.mp.brier.hier}, which only just
excludes zero, and \npj{PH.X2.grpocont_vs_s1cont_06b.qwen3-0.6b.mp.aurc.hier}), and there too the raw readout moved
towards certainty (\npj{arm.qwen3-0.6b.mp.fork_grpo.satwrong}\% against \npj{arm.qwen3-0.6b.mp.s1cont.satwrong}\%), as
on Qwen3-1.7B (Results).
\item \emph{GRPO continuation minus GRPO+TS (Qwen3-0.6B; 800 against 400 steps, so not budget-matched).} The Brier
score and AURC were lower on both tasks (GSM8K-Verify \npj{PH.X2.grpocont_vs_grpo_06b.qwen3-0.6b.gv.brier.diff} and
\npj{PH.X2.grpocont_vs_grpo_06b.qwen3-0.6b.gv.aurc.diff}, MMLU-Pro
\npj{PH.X2.grpocont_vs_grpo_06b.qwen3-0.6b.mp.brier.diff} and \npj{PH.X2.grpocont_vs_grpo_06b.qwen3-0.6b.mp.aurc.diff};
each Holm-adjusted \npj{PH.X2.grpocont_vs_grpo_06b.qwen3-0.6b.gv.aurc.p}, in every seed, with every seed-and-item
interval excluding zero); accuracy was higher on MMLU-Pro (\npj{PH.X2.grpocont_vs_grpo_06b.qwen3-0.6b.mp.acc.pp}
points, \npj{PH.X2.grpocont_vs_grpo_06b.qwen3-0.6b.mp.acc.p}) and did not differ on GSM8K-Verify
(\npj{PH.X2.grpocont_vs_grpo_06b.qwen3-0.6b.gv.acc.pp} points, \npj{PH.X2.grpocont_vs_grpo_06b.qwen3-0.6b.gv.acc.p}).
\item \emph{Stage-1 continuation minus stage~1 (not in the archived exploratory analysis plan).} No test differed (Holm-adjusted
\npj{PH.X2.s1cont_vs_stage1.qwen3-0.6b.mp.acc.p} or larger; Qwen3-0.6B GSM8K-Verify accuracy
\npj{PH.X2.s1cont_vs_stage1.qwen3-0.6b.gv.acc.pp} points).
\item \emph{Rationales (descriptive).} In training, the continuations that rewarded the reasoning moved it further from
the stage-1 checkpoint, the reference of every continuation (Qwen3-0.6B GSM8K-Verify, mean token KL: two-stage
\npj{arm.qwen3-0.6b.gv.fork_rl.kl_tokens}, GRPO continuation \npj{arm.qwen3-0.6b.gv.fork_grpo.kl_tokens}, stage-1
continuation \npj{arm.qwen3-0.6b.gv.s1cont.kl_tokens}), at about the same length of the test rationales
(\npj{arm.qwen3-0.6b.gv.fork_rl.mean_rationale_tokens}, \npj{arm.qwen3-0.6b.gv.fork_grpo.mean_rationale_tokens} and
\npj{arm.qwen3-0.6b.gv.s1cont.mean_rationale_tokens} tokens; per run in the archived run-level results).
\item \emph{A possible reason, not tested here.} On Qwen3-0.6B GSM8K-Verify, the continuations that rewarded the
reasoning were more accurate than the stage-1 continuation, which itself did not differ from stage~1, and removing the
anchor raised accuracy at stage~1's own budget (X3:
\npj{PH.X3.own_noanchor_vs_anchored.qwen3-0.6b.gv.acc.pp} points over the anchored stage~1,
\npj{PH.X3.own_noanchor_vs_anchored.qwen3-0.6b.gv.acc.p}). The anchored, unrewarded reasoning, rather than the proper
score on the readout, which the two-stage model and the arm without the anchor both kept, may therefore limit stage~1's
accuracy in this cell; neither experiment shows whether the reasoning itself improved or the readout used it better.
The two-stage model's gain was not robust and appeared in no other cell, so the second stage remains optional.
\item One X2 run (Qwen3-1.7B MMLU-Pro, stage-1 continuation, seed 29) stopped after step 180, the last step logged,
when the machine restarted at 16:06 on 30 October; it was rerun from the start with the same command, and the rerun is
the reported run. No X2 run was excluded.
\end{itemize}

\paragraph{X3: rationales free to move.} Under the token anchor, the model's own rationales stayed close to the initial
model's (Supplementary Table~3), so a readout trained after them and one trained after offline rationales could hardly
differ, which might explain why Q1b was inconclusive. Without the anchor, readout training can move the model's own
rationales, and a readout trained after offline rationales then meets, at evaluation, rationales unlike those it was
trained after.
\begin{itemize}
\item \emph{Design.} The stage-1 recipe of Q1b (Brier score, 400 steps, seeds 17, 23 and 29, the four Qwen3 cells)
without the token anchor, its training settings otherwise unchanged: the readout was trained after the model's own
rationales or after offline rationales sampled at every step from the frozen initial model. As in every arm,
evaluation used the model's own rationales. The pre-specified runs with the anchor are the counterparts.
\item \emph{Prediction.} Own rationales give a lower Brier score and AURC than offline ones (single query after
temperature scaling; Holm over four cells $\times$ two metrics). Also reported: each arm minus its anchored counterpart
(Holm over four cells $\times$ three metrics, accuracy included; no seed-and-item interval) and the rationale
statistics.
\item \emph{The rationales moved.} In the training batches logged every tenth step over the last 100 steps, the
answer marker had to be appended to \npj{PH.X3.trainforced.qwen3-0.6b.gv.noanchor}\%, \npj{PH.X3.trainforced.qwen3-1.7b.gv.noanchor}\%
and \npj{PH.X3.trainforced.qwen3-1.7b.mp.noanchor}\% of the own arm's training rationales on Qwen3-0.6B GSM8K-Verify,
Qwen3-1.7B GSM8K-Verify and Qwen3-1.7B MMLU-Pro (with the anchor, \npj{PH.X3.trainforced.qwen3-0.6b.gv.warm400}\%,
\npj{PH.X3.trainforced.qwen3-1.7b.gv.warm400}\% and \npj{PH.X3.trainforced.qwen3-1.7b.mp.warm400}\%), while the offline
arm's readout was trained after the initial model's rationales (\npj{PH.X3.trainforced.qwen3-0.6b.gv.noanchor_refrat}\%,
\npj{PH.X3.trainforced.qwen3-1.7b.gv.noanchor_refrat}\% and \npj{PH.X3.trainforced.qwen3-1.7b.mp.noanchor_refrat}\%).
At evaluation, both arms' own rationales lacked the marker far more often than with the anchor (own arm
\npj{PH.X3.rationale.qwen3-0.6b.gv.noanchor.forced}\%, \npj{PH.X3.rationale.qwen3-1.7b.gv.noanchor.forced}\% and
\npj{PH.X3.rationale.qwen3-1.7b.mp.noanchor.forced}\%; offline arm
\npj{PH.X3.rationale.qwen3-0.6b.gv.noanchor_refrat.forced}\%,
\npj{PH.X3.rationale.qwen3-1.7b.gv.noanchor_refrat.forced}\% and
\npj{PH.X3.rationale.qwen3-1.7b.mp.noanchor_refrat.forced}\%; with the anchor
\npj{PH.X3.rationale.qwen3-0.6b.gv.warm400.forced}\%, \npj{PH.X3.rationale.qwen3-1.7b.gv.warm400.forced}\% and
\npj{PH.X3.rationale.qwen3-1.7b.mp.warm400.forced}\%), so the offline arm's readout was evaluated after rationales
unlike those it had been trained after.
\item On Qwen3-0.6B MMLU-Pro, the offline arm's rationales lacked the marker about as often as with the anchor
(\npj{PH.X3.rationale.qwen3-0.6b.mp.noanchor_refrat.forced}\% at evaluation;
\npj{PH.X3.rationale.qwen3-0.6b.mp.refrat.forced}\% with the anchor) but reached the token cap about half as often
(\npj{PH.X3.rationale.qwen3-0.6b.mp.noanchor_refrat.length_limit}\% against
\npj{PH.X3.rationale.qwen3-0.6b.mp.refrat.length_limit}\%). In the own arm, seed 23 stopped reasoning, writing little
more than an option letter (mean \npj{PH.X3.rationale.qwen3-0.6b.mp.noanchor.s23.tokens} tokens, marker appended to
\npj{PH.X3.rationale.qwen3-0.6b.mp.noanchor.s23.forced}\%); in seeds 17 and 29 the marker was appended to
\npj{PH.X3.rationale.qwen3-0.6b.mp.noanchor.s17.forced}\% and \npj{PH.X3.rationale.qwen3-0.6b.mp.noanchor.s29.forced}\%
(with the anchor, \npj{PH.X3.rationale.qwen3-0.6b.mp.warm400.s17.forced}\% and
\npj{PH.X3.rationale.qwen3-0.6b.mp.warm400.s29.forced}\%) and the mean length was
\npj{PH.X3.rationale.qwen3-0.6b.mp.noanchor.s17.tokens} and \npj{PH.X3.rationale.qwen3-0.6b.mp.noanchor.s29.tokens}
tokens (with the anchor, \npj{PH.X3.rationale.qwen3-0.6b.mp.warm400.s17.tokens} and
\npj{PH.X3.rationale.qwen3-0.6b.mp.warm400.s29.tokens}). Seed 23 was kept, as the exclusion rule requires (mean length
above five tokens); no X3 run was excluded. Without the anchor, the token KL of Supplementary Table~3 is not computed.
\item \emph{Result.} The prediction held in one of the eight tests, not robustly. On Qwen3-1.7B MMLU-Pro, own
rationales ranked better (AURC, own minus offline, \npj{PH.X3.own_vs_offline_noanchor.qwen3-1.7b.mp.aurc.diff}, 95\% CI
\npj{PH.X3.own_vs_offline_noanchor.qwen3-1.7b.mp.aurc.ci}, Holm-adjusted
\npj{PH.X3.own_vs_offline_noanchor.qwen3-1.7b.mp.aurc.p}, \npj{PH.X3.own_vs_offline_noanchor.qwen3-1.7b.mp.aurc.seeds}
seeds), but the seed-and-item interval (\npj{PH.X3.own_vs_offline_noanchor.qwen3-1.7b.mp.aurc.hier}) included zero,
and the Brier score did not differ after adjustment (\npj{PH.X3.own_vs_offline_noanchor.qwen3-1.7b.mp.brier.diff},
Holm-adjusted \npj{PH.X3.own_vs_offline_noanchor.qwen3-1.7b.mp.brier.p}). In that cell the offline arm's rationales
reached the token cap more often (\npj{PH.X3.rationale.qwen3-1.7b.mp.noanchor_refrat.length_limit}\% against
\npj{PH.X3.rationale.qwen3-1.7b.mp.noanchor.length_limit}\%), so own rationales may have ranked better because the
offline arm's decisions more often followed truncated reasoning, not because its readout met rationales unlike those
it was trained after; X3 cannot tell the two apart. The other six tests showed no difference in either direction
(\npj{PH.X3.own_vs_offline_noanchor.qwen3-1.7b.gv.brier.diff} to
\npj{PH.X3.own_vs_offline_noanchor.qwen3-1.7b.gv.aurc.diff}; Holm-adjusted
\npj{PH.X3.own_vs_offline_noanchor.qwen3-1.7b.gv.aurc.p} or larger).
\item Without the anchor, as with it, the model's own rationales brought no consistent gain, although the offline arm's
readout met at evaluation rationales unlike those it was trained after. That the anchor kept the two alike therefore
does not seem to explain why Q1b was inconclusive.
\item \emph{Against the anchored arms (reported).} Removing the anchor helped on Qwen3-0.6B GSM8K-Verify, after own and
after offline rationales: accuracy \npj{PH.X3.own_noanchor_vs_anchored.qwen3-0.6b.gv.acc.pp} and
\npj{PH.X3.offline_noanchor_vs_anchored.qwen3-0.6b.gv.acc.pp} percentage points, Brier score
\npj{PH.X3.own_noanchor_vs_anchored.qwen3-0.6b.gv.brier.diff} and
\npj{PH.X3.offline_noanchor_vs_anchored.qwen3-0.6b.gv.brier.diff}, AURC
\npj{PH.X3.own_noanchor_vs_anchored.qwen3-0.6b.gv.aurc.diff} and
\npj{PH.X3.offline_noanchor_vs_anchored.qwen3-0.6b.gv.aurc.diff} (each Holm-adjusted
\npj{PH.X3.offline_noanchor_vs_anchored.qwen3-0.6b.gv.aurc.p} or smaller; every seed agreed). On Qwen3-1.7B
GSM8K-Verify it lowered the AURC after offline rationales
(\npj{PH.X3.offline_noanchor_vs_anchored.qwen3-1.7b.gv.aurc.diff},
\npj{PH.X3.offline_noanchor_vs_anchored.qwen3-1.7b.gv.aurc.p},
\npj{PH.X3.offline_noanchor_vs_anchored.qwen3-1.7b.gv.aurc.seeds} seeds; after own rationales
\npj{PH.X3.own_noanchor_vs_anchored.qwen3-1.7b.gv.aurc.diff},
\npj{PH.X3.own_noanchor_vs_anchored.qwen3-1.7b.gv.aurc.p}). On Qwen3-0.6B MMLU-Pro it hurt after own rationales
(Brier score \npj{PH.X3.own_noanchor_vs_anchored.qwen3-0.6b.mp.brier.diff},
\npj{PH.X3.own_noanchor_vs_anchored.qwen3-0.6b.mp.brier.p}; AURC
\npj{PH.X3.own_noanchor_vs_anchored.qwen3-0.6b.mp.aurc.diff},
\npj{PH.X3.own_noanchor_vs_anchored.qwen3-0.6b.mp.aurc.p}; every seed agreed). No other comparison differed.
\item \emph{Descriptive, not tested.} On Qwen3-0.6B GSM8K-Verify, stage~1 without the anchor reached
\npj{arm.qwen3-0.6b.gv.noanchor.single.ts.acc.pct}\% accuracy (with it,
\npj{arm.qwen3-0.6b.gv.warm400.single.ts.acc.pct}\%; GRPO+TS, which kept the anchor,
\npj{arm.qwen3-0.6b.gv.grpo.single.ts.acc.pct}\%). On Qwen3-1.7B GSM8K-Verify, where the anchored readout saturated
(Discussion), \npj{arm.qwen3-1.7b.gv.noanchor.satwrong}\% of wrong answers had a raw confidence of at least 0.99
without the anchor (a mean over three seeds, one of which stayed largely saturated), against \npj{arm.qwen3-1.7b.gv.warm400.satwrong}\% with it (fitted temperature
\npj{arm.qwen3-1.7b.gv.noanchor.single.T} against \npj{arm.qwen3-1.7b.gv.warm400.single.T}).
\item One X3 run (Qwen3-1.7B MMLU-Pro, own rationales, seed 17) stopped after step 220, the last step logged, with no
error in its log, together with the job-queue worker that ran it, most likely killed by the operating system under
memory pressure; it was rerun from the start with the same command, and the rerun is the reported run.
\end{itemize}

\paragraph{Additional verification seeds.} The accuracy deficit of PRDO against GRPO+TS on Qwen3-0.6B
GSM8K-Verify came mostly from one seed (seed 29; the archived run-level results). An exploratory follow-up added seeds 41 and 53
of both arms, together with their SVAMP-Verify evaluations; the analysis was fixed before these additional runs. It
compares the arms over the five seeds in accuracy, Brier score and AURC after temperature scaling (Holm over the
three) and in SVAMP-Verify accuracy and AURC with the source run's temperature (Holm over the two), with per-seed
signs and the seed-and-item bootstrap; it makes no prediction.
\begin{itemize}
\item \emph{Result, GSM8K-Verify.} Over the five seeds, PRDO was still less accurate than
GRPO+TS (\npj{PH.X4.five_seeds_gv.qwen3-0.6b.gv.acc.pp} percentage points, 95\% CI
\npj{PH.X4.five_seeds_gv.qwen3-0.6b.gv.acc.cipp}, Holm-adjusted \npj{PH.X4.five_seeds_gv.qwen3-0.6b.gv.acc.p},
\npj{PH.X4.five_seeds_gv.qwen3-0.6b.gv.acc.seeds} seeds), though not robustly (seed-and-item interval
\npj{PH.X4.five_seeds_gv.qwen3-0.6b.gv.acc.hierpp} points; mean accuracy \npj{PH.X4.mean.gv.warm400.acc}\% against
\npj{PH.X4.mean.gv.grpo.acc}\%). Its AURC stayed lower (\npj{PH.X4.five_seeds_gv.qwen3-0.6b.gv.aurc.diff},
\npj{PH.X4.five_seeds_gv.qwen3-0.6b.gv.aurc.p}, \npj{PH.X4.five_seeds_gv.qwen3-0.6b.gv.aurc.seeds} seeds, seed-and-item
interval \npj{PH.X4.five_seeds_gv.qwen3-0.6b.gv.aurc.hier}), and its Brier score did not differ
(\npj{PH.X4.five_seeds_gv.qwen3-0.6b.gv.brier.diff}, \npj{PH.X4.five_seeds_gv.qwen3-0.6b.gv.brier.p}).
\item \emph{Result, SVAMP-Verify} (the same weights, with the source runs' temperatures). The accuracy deficit was
larger and robust (\npj{PH.X4.five_seeds_svampv.qwen3-0.6b.svampv.acc.pp} percentage points, 95\% CI
\npj{PH.X4.five_seeds_svampv.qwen3-0.6b.svampv.acc.cipp}, Holm-adjusted
\npj{PH.X4.five_seeds_svampv.qwen3-0.6b.svampv.acc.p}, \npj{PH.X4.five_seeds_svampv.qwen3-0.6b.svampv.acc.seeds} seeds;
seed-and-item interval \npj{PH.X4.five_seeds_svampv.qwen3-0.6b.svampv.acc.hierpp} points), and the AURC again lower
(\npj{PH.X4.five_seeds_svampv.qwen3-0.6b.svampv.aurc.diff}, \npj{PH.X4.five_seeds_svampv.qwen3-0.6b.svampv.aurc.p},
\npj{PH.X4.five_seeds_svampv.qwen3-0.6b.svampv.aurc.seeds} seeds).
\item Compared with the three original three seeds, the deficit moved towards zero on GSM8K-Verify (from
\npj{C.sec.qwen3-0.6b.gv.acc.pp} points) but not on SVAMP-Verify (from \npj{Q3b.sec.qwen3-0.6b.svampv.acc.pp}).
\item \emph{Seeds 41 and 53 (descriptive; not part of the pre-specified seed-extension analysis).} Accuracy differences
\npj{PH.X4.seed41.gv.acc} and \npj{PH.X4.seed53.gv.acc} points on GSM8K-Verify, \npj{PH.X4.seed41.svampv.acc} and
\npj{PH.X4.seed53.svampv.acc} on SVAMP-Verify: one new seed showed no deficit on GSM8K-Verify, the other a large one on
both tasks, so the deficit did not come from seed 29 alone. All eight X4 runs finished; none was excluded or rerun.
\end{itemize}

\section{Prompts}\label{app:prompts}

Every model received one user turn in its chat template, in non-thinking mode unless stated (Section~\ref{sec:experimental-setup}); the template
date of SmolLM3-3B was fixed at 05 October 2026. The rationale is generated after the turn; the readout is taken after
the marker \texttt{Answer:}, over the tokens of the options. The exact strings are in the code
(\texttt{openjev\_rlcd/tasks}).

\paragraph{MMLU-Pro and ARC-Challenge.} \texttt{Question:} the stem; \texttt{Options:} one line per option,
lettered \texttt{A.}, \texttt{B.}, \ldots; then: \emph{Think step by step, briefly (a few sentences), then finish with exactly one
line of the form `Answer: X', where X is the letter of the correct option.}

\paragraph{GSM8K-Verify and SVAMP-Verify.} \texttt{Problem:} the problem; \texttt{Proposed final answer:} the
proposal; then: \emph{Is the proposed final answer correct? Check it: think step by step, briefly, then finish with
exactly one line of the form `Answer: yes' or `Answer: no'.}

\paragraph{The solver behind the proposals.} Qwen3-0.6B received the problem followed by: \emph{Solve the problem.
Think step by step briefly, then finish with exactly one line of the form `Answer: <number>'.}

\paragraph{ChaosNLI.} \texttt{Premise:} and \texttt{Hypothesis:} lines; then: \emph{Does the premise entail the
hypothesis, contradict it, or neither (neutral)? Think briefly in at most three sentences, then finish with exactly
one line of the form `Answer: entailment', `Answer: neutral', or `Answer: contradiction'.}

\paragraph{SFT+TS.} The same prompt as the other arms, which asks for reasoning; the answer starts at the marker,
with an empty rationale.

\paragraph{RLCR+TS (reinforcement learning with calibration rewards)} appends to the turn: \emph{After that line, write one more line of the form `Confidence: C', where
C is your probability (a number from 0 to 1) that your answer is correct.}

\paragraph{CREDO+TS} appends: \emph{After that line, analyze how likely your answer is to be correct between
\texttt{<analysis>} and \texttt{</analysis>}, then write \texttt{<confidence>high</confidence>} if your answer is
likely correct or \texttt{<confidence>low</confidence>} if it is likely wrong.}

\clearpage
\begin{figure}[p]
\centering
\includegraphics[width=\textwidth]{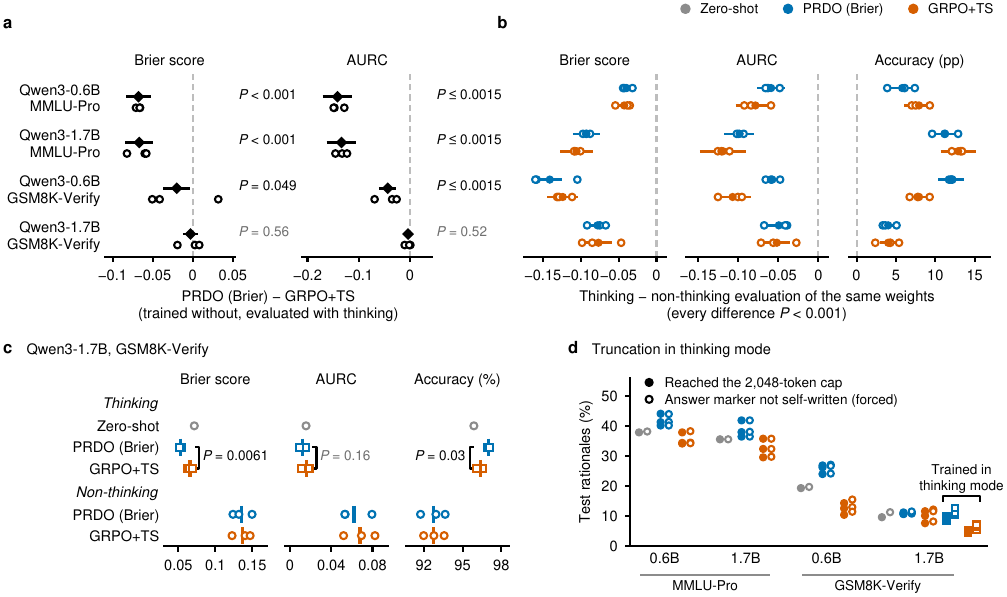}
\caption{\textbf{PRDO in capped thinking mode.}
Single-query readouts after temperature scaling on MMLU-Pro (1,000 questions) and GSM8K-Verify (1,260), with a
2,048-token thinking cap. \textbf{a}, PRDO minus GRPO+TS when weights trained without thinking
are evaluated with thinking. Negative Brier-score and AURC differences favour PRDO.
\textbf{b}, Thinking minus non-thinking evaluation of the same weights. \textbf{c}, Qwen3-1.7B on GSM8K-Verify when
both methods are trained and evaluated with thinking (squares), compared with their non-thinking versions (circles).
\textbf{d}, Share of thinking-mode rationales that reach the token cap (filled) or lack a self-written answer marker
(open). Diamonds or filled circles show differences of three-seed means, lines are 95\% item-bootstrap confidence
intervals and open symbols are seeds. PRDO weights reached the cap more often than GRPO+TS in every cell,
so this descriptive comparison mixes readout quality with unequal truncation.}
\label{fig:supp-thinking}
\end{figure}
\clearpage

\begin{figure}[p]
\centering
\includegraphics[width=\textwidth]{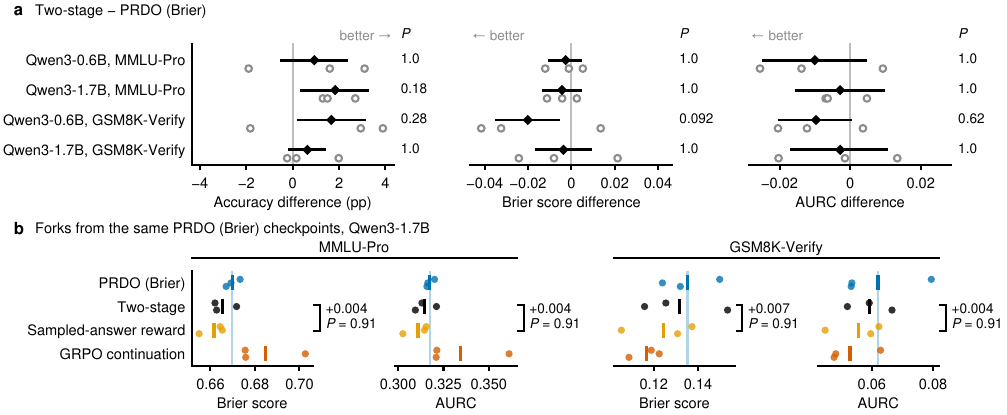}
\caption{\textbf{Additional rationale reinforcement does not yield a detectable improvement over PRDO alone.}
\textbf{a}, Difference between two-stage training and its stage-1 checkpoint on the two Qwen3 models and two tasks.
The second stage adds 400 updates that reward each rationale with the Brier score of its readout; it therefore uses
twice the total training budget of stage 1. Diamonds are mean differences, bars are 95\% confidence intervals and
circles are paired seeds. None of the twelve Holm-adjusted comparisons is significant. \textbf{b}, Qwen3-1.7B
stage-1 checkpoints and three continuations: the proper Brier reward on each rationale, a sampled-answer correctness
reward and GRPO. Brackets show the pre-specified proper-reward versus sampled-answer comparison; the proper reward is
not better on either task. Brier score and AURC are computed from the temperature-scaled single-rationale readout.
Each arm has three seeds; test sets contain 1,000 MMLU-Pro and 1,260 GSM8K-Verify items.}
\label{fig:supp-stage2}
\end{figure}
\clearpage

\begin{figure}[p]
\centering
\includegraphics[width=\textwidth]{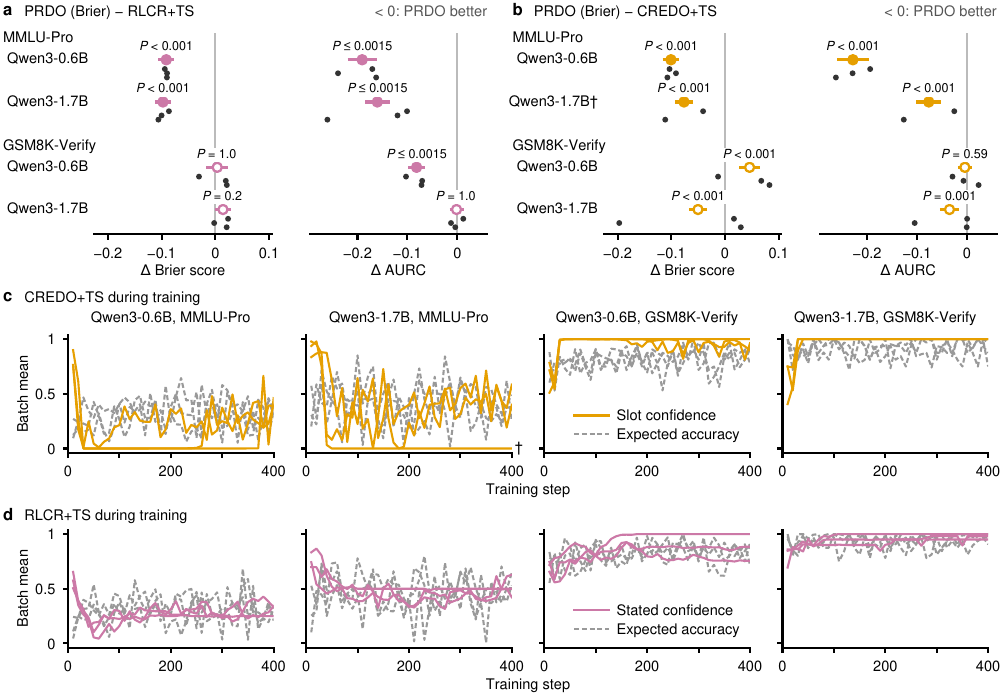}
\caption{\textbf{Comparison with methods that produce a separate confidence output.}
\textbf{a,b}, Difference between PRDO and RLCR+TS (reinforcement learning with calibration rewards on a stated
confidence; a) or CREDO+TS (confidence read from a dedicated slot after self-analysis; b) in Brier score and AURC
on the final test sets. Negative values favour PRDO. Circles show differences of arm means,
coloured lines are 95\% item-bootstrap confidence intervals and grey dots are same-seed differences. Open circles
mark comparisons whose seed-and-item interval includes zero. Each arm has three seeds except CREDO+TS on
Qwen3-1.7B MMLU-Pro, where one run was excluded by the collapse rule. The RLCR comparison is post hoc; CREDO was run
untuned under settings fixed before its runs. \textbf{c,d}, Training dynamics of CREDO's confidence slot (c) and
RLCR's stated confidence (d), with expected sampled-answer accuracy in grey. One line represents one seed. CREDO's
slot saturated early on GSM8K-Verify and was unstable on MMLU-Pro; RLCR's stated confidence settled near training
accuracy. No central claim depends on these comparisons.}
\label{fig:supp-closest}
\end{figure}

\clearpage
\footnotesize
\setlength{\tabcolsep}{3pt}
\begin{landscape}
\section{Additional Tables}
\begin{longtable}{@{}>{\raggedright\arraybackslash}p{5.0cm}>{\raggedright\arraybackslash}p{4.3cm}>{\raggedright\arraybackslash}p{2.9cm}l rrcl@{}}
\caption{\textbf{Every tested comparison, pre-specified or post hoc.} Difference = first arm minus second (Brier, AURC, E-AURC, NLL, JSD and squared distance: negative favours the first arm; accuracy and coverage: positive favours it), averaged over seeds on the same test items, with its 95\% confidence interval. Per-item metrics: two-sided sign-flip permutation test (20,000 flips); AURC, E-AURC and oracle coverage: two-sided item bootstrap (4,000 resamples). $P_{\mathrm{Holm}}$: Holm-adjusted within the family named in the first column. Seeds: seeds whose own difference has the sign of the mean difference. Hierarchical: 95\% interval of a bootstrap over seeds and items (robustness check). $\leq$: the unadjusted $P$ sits at its Monte Carlo floor, so the adjusted value is an upper bound (Appendix~\ref{app:statistics}). Each deciding family is followed by its secondary endpoints (accuracy, NLL, E-AURC; each metric a Holm family of its own); those of Q3a, which the pre-specified decision code omitted, were computed after the decisions with its own functions and are labelled post hoc (the post-hoc analysis). The comparison with RLCR+TS was specified after its runs had finished and is post hoc. ChaosNLI (a replication on previously read test items) and the thinking-mode analyses are reported descriptively. The last block, labelled post hoc, holds the exploratory experiments of Appendix~\ref{app:exploratory} (the exploratory analyses), each family with its own Holm adjustment over all the cells and metrics it lists, accuracy included.}\label{stab:decisions}\\
\toprule Family & Model, task & Metric & Difference [95\% CI] & $P$ & $P_{\mathrm{Holm}}$ & Seeds & Hierarchical \\ \midrule \endfirsthead
\toprule Family & Model, task & Metric & Difference [95\% CI] & $P$ & $P_{\mathrm{Holm}}$ & Seeds & Hierarchical \\ \midrule \endhead
\bottomrule \endfoot
C: PRDO vs GRPO+TS & Phi-4-mini, MMLU-Pro & AURC & $-$0.106 [$-$0.131, $-$0.082] & $<$0.001 & $\leq$0.002 & 3/3 & $-$0.141 to $-$0.072 \\
 & Phi-4-mini, MMLU-Pro & Brier & $-$0.068 [$-$0.084, $-$0.052] & $<$0.001 & $<$0.001 & 3/3 & $-$0.093 to $-$0.042 \\
 & Qwen3-0.6B, GSM8K-Verify & AURC & $-$0.093 [$-$0.117, $-$0.068] & $<$0.001 & $\leq$0.002 & 3/3 & $-$0.124 to $-$0.062 \\
 & Qwen3-0.6B, GSM8K-Verify & Brier & $-$0.003 [$-$0.022, 0.016] & 0.76 & 1.0 & 2/3 & $-$0.031 to 0.027 \\
 & Qwen3-0.6B, MMLU-Pro & AURC & $-$0.161 [$-$0.193, $-$0.130] & $<$0.001 & $\leq$0.002 & 3/3 & $-$0.208 to $-$0.118 \\
 & Qwen3-0.6B, MMLU-Pro & Brier & $-$0.071 [$-$0.085, $-$0.057] & $<$0.001 & $<$0.001 & 3/3 & $-$0.092 to $-$0.052 \\
 & Qwen3-1.7B, GSM8K-Verify & AURC & $-$0.006 [$-$0.022, 0.010] & 0.48 & 1.0 & 2/3 & $-$0.033 to 0.022 \\
 & Qwen3-1.7B, GSM8K-Verify & Brier & $-$0.001 [$-$0.016, 0.013] & 0.84 & 1.0 & 1/3 & $-$0.026 to 0.024 \\
\end{longtable}

\clearpage
\begin{longtable}{@{}>{\raggedright\arraybackslash}p{5.4cm}cccccccccc@{}}
\caption{\textbf{Every arm on the final test set (single query), apart from the post-hoc experiments.} Their runs (Appendix~\ref{app:exploratory}) are in the archived run-level results. Mean $\pm$ s.d. over the seeds kept by the pre-specified exclusion rule ($n$; the zero-shot model is one run). Accuracy and oracle coverage at error 0.2 (Cov.) in \%; Brier score, NLL, ECE, AURC, E-AURC and coverage after temperature scaling (T fitted on the training task's dev split); ECE$_\mathrm{raw}$ before scaling; T the fitted temperature (T $>$ 1: the raw readout was overconfident). Transfer tasks (ARC-Challenge, SVAMP-Verify): every arm under the source run's T and, in the row below it, under a T fitted on the new task's development questions.}\label{stab:arms}\\
\toprule Arm & $n$ & Accuracy & Brier & NLL & ECE & AURC & E-AURC & Cov. & ECE$_\mathrm{raw}$ & T \\ \midrule \endfirsthead
\toprule Arm & $n$ & Accuracy & Brier & NLL & ECE & AURC & E-AURC & Cov. & ECE$_\mathrm{raw}$ & T \\ \midrule \endhead
\bottomrule \endfoot
\multicolumn{11}{@{}l}{\textbf{Qwen3-0.6B, MMLU-Pro}} \\
Zero-shot & 1 & 28.0 & 0.852 & 2.078 & 0.063 & 0.660 & 0.296 & 0.1 & 0.448 & 5.19 \\
SFT+TS & 3 & 34.2$\pm$1.1 & 0.767$\pm$0.001 & 1.827$\pm$0.008 & 0.025$\pm$0.008 & 0.473$\pm$0.003 & 0.182$\pm$0.010 & 9.7$\pm$1.5 & 0.176$\pm$0.009 & 1.77$\pm$0.11 \\
GRPO+TS & 3 & 31.9$\pm$0.9 & 0.827$\pm$0.007 & 2.015$\pm$0.022 & 0.044$\pm$0.012 & 0.605$\pm$0.020 & 0.288$\pm$0.013 & 0.4$\pm$0.4 & 0.614$\pm$0.025 & 8.61$\pm$1.63 \\
PRDO (Brier) & 3 & 35.5$\pm$2.3 & 0.757$\pm$0.007 & 1.819$\pm$0.017 & 0.030$\pm$0.007 & 0.444$\pm$0.013 & 0.166$\pm$0.010 & 12.7$\pm$1.6 & 0.120$\pm$0.012 & 1.53$\pm$0.14 \\
PRDO (log) & 3 & 35.0$\pm$0.9 & 0.758$\pm$0.002 & 1.804$\pm$0.005 & 0.032$\pm$0.009 & 0.454$\pm$0.008 & 0.170$\pm$0.011 & 11.9$\pm$1.8 & 0.185$\pm$0.018 & 1.71$\pm$0.11 \\
Linear score & 3 & 34.3$\pm$1.3 & 0.807$\pm$0.005 & 1.968$\pm$0.016 & 0.041$\pm$0.015 & 0.548$\pm$0.017 & 0.257$\pm$0.007 & 0.5$\pm$0.6 & 0.564$\pm$0.038 & 8.55$\pm$0.95 \\
Offline rationales & 3 & 35.5$\pm$1.7 & 0.762$\pm$0.009 & 1.838$\pm$0.017 & 0.029$\pm$0.004 & 0.450$\pm$0.020 & 0.172$\pm$0.008 & 10.7$\pm$2.4 & 0.101$\pm$0.031 & 1.47$\pm$0.18 \\
\end{longtable}

\end{landscape}
\clearpage
\begin{landscape}
\begin{longtable}{@{}>{\raggedright\arraybackslash}p{3.3cm}>{\raggedright\arraybackslash}p{5.2cm}cccccc@{}}
\caption{\textbf{Rationale statistics.} Mean $\pm$ s.d. over the seeds kept. On the final test set: forced marker, the share of rationales that did not write the answer marker themselves (the marker is then appended, also after the token cap); at cap, the share that reached the token cap (256 MMLU-Pro, 320 GSM8K-Verify, 128 ChaosNLI, 2,048 in thinking mode); tokens, the mean rationale length; repetition, the repetition ratio that enters the collapse rule. In training: token KL, the per-rationale sum of next-token KL divergences to the run's reference model (nats): the initial model or, for the arms forked from a stage-1 checkpoint (two-stage, sampled-answer reward, GRPO continuation), that checkpoint; logged every 10 steps and averaged over steps and seeds; it is what the token anchor penalizes, which GRPO+TS carries too; --, not computed (no rationale, evaluation only, or no token anchor: RLCR+TS and CREDO+TS). PRDO's token KL was \npj{klratio.min}--\npj{klratio.max} times GRPO+TS's in the four Qwen3 cells.}\label{stab:rationale}\\
\toprule Model, task & Arm & $n$ & Forced marker (\%) & At cap (\%) & Tokens & Repetition & Token KL \\ \midrule \endfirsthead
\toprule Model, task & Arm & $n$ & Forced marker (\%) & At cap (\%) & Tokens & Repetition & Token KL \\ \midrule \endhead
\bottomrule \endfoot
Qwen3-0.6B, MMLU-Pro & Zero-shot & 1 & 15.5 & 15.2 & 118 & 0.26 & -- \\
 & SFT+TS & 3 & -- & -- & -- & -- & -- \\
 & GRPO+TS & 3 & 20.8$\pm$4.2 & 20.5$\pm$3.7 & 123$\pm$10 & 0.26$\pm$0.01 & 3.45$\pm$0.19 \\
 & PRDO (Brier) & 3 & 16.5$\pm$1.3 & 16.2$\pm$1.2 & 122$\pm$4 & 0.27$\pm$0.01 & 1.13$\pm$0.07 \\
 & PRDO (log) & 3 & 16.2$\pm$2.3 & 16.2$\pm$2.1 & 121$\pm$7 & 0.27$\pm$0.01 & 1.96$\pm$0.09 \\
 & Linear score & 3 & 16.4$\pm$3.5 & 16.3$\pm$3.5 & 122$\pm$10 & 0.27$\pm$0.01 & 1.37$\pm$0.10 \\
 & Offline rationales & 3 & 15.5$\pm$1.7 & 15.3$\pm$1.6 & 119$\pm$4 & 0.27$\pm$0.00 & 1.18$\pm$0.04 \\
 & Mixture objective & 3 & 15.0$\pm$0.8 & 14.8$\pm$0.7 & 117$\pm$1 & 0.26$\pm$0.00 & 1.09$\pm$0.03 \\
\end{longtable}

\end{landscape}
\clearpage
\begin{longtable}{@{}>{\raggedright\arraybackslash}p{6.2cm}lll@{}}
\caption{\textbf{Risk-controlled thresholds, chosen on a held-out calibration split and applied to the final test set (reported descriptively).} The threshold is the lowest value on a grid fixed in advance (0.30 to 0.999, 19 values) for which the calibration-split questions answered at or above it (scaled single-query confidence at least the threshold) have a one-sided Clopper--Pearson upper bound on their error of at most $\epsilon$ at level $\delta/19$ (Bonferroni over the grid, $\delta = 0.1$: with probability at least 0.9 every grid threshold that meets this bound has selective error at most $\epsilon$); --: no threshold passed, so the system answers nothing. Coverage and realized error on the final test set, per seed (17 / 23 / 29; the zero-shot model is one run). A coverage of 100\% means the lowest threshold accepted every question.}\label{stab:riskcontrol}\\
\toprule Model, task, arm & $\epsilon$ & Coverage (\%) & Realized error (\%) \\ \midrule \endfirsthead
\toprule Model, task, arm & $\epsilon$ & Coverage (\%) & Realized error (\%) \\ \midrule \endhead
\bottomrule \endfoot
Qwen3-0.6B, MMLU-Pro, Zero-shot & 0.2 & 0.0 & -- \\
Qwen3-0.6B, MMLU-Pro, Zero-shot & 0.1 & 0.0 & -- \\
Qwen3-0.6B, MMLU-Pro, SFT+TS & 0.2 & 0.0 / 0.0 / 0.0 & -- / -- / -- \\
Qwen3-0.6B, MMLU-Pro, SFT+TS & 0.1 & 0.0 / 0.0 / 0.0 & -- / -- / -- \\
Qwen3-0.6B, MMLU-Pro, GRPO+TS & 0.2 & 0.0 / 0.0 / 0.0 & -- / -- / -- \\
Qwen3-0.6B, MMLU-Pro, GRPO+TS & 0.1 & 0.0 / 0.0 / 0.0 & -- / -- / -- \\
Qwen3-0.6B, MMLU-Pro, PRDO (Brier) & 0.2 & 3.8 / 0.0 / 0.0 & 10.5 / -- / -- \\
Qwen3-0.6B, MMLU-Pro, PRDO (Brier) & 0.1 & 0.0 / 0.0 / 0.0 & -- / -- / -- \\
\end{longtable}

\begin{landscape}
\begin{table}[p]\centering\footnotesize
\caption{\textbf{Every CREDO+TS run.} Forced marker: share of test rationales that did not write the answer marker themselves (it is reported per run). Slot readout on the test set (single query): median logit gap $g = z_{\mathrm{high}} - z_{\mathrm{low}}$ and the share of questions with $g > 10$ (confidence above 0.99995) or $g < -10$. Accuracy, multi-class Brier score and AURC after temperature scaling with the fitted temperature $T$. TL Brier, the top-label Brier score of the confidence after scaling, with its Murphy decomposition (10 equal-width bins): reliability (REL), resolution (RES) and uncertainty (UNC), TL Brier $\approx$ REL $-$ RES $+$ UNC (this is reported because the multi-class Brier score of the slot readout includes the even spread of $1-c$ over the other options). Below each cell, the means over the kept seeds of PRDO (Brier) and RLCR+TS, for reference. Excluded: removed by the pre-specified rationale-collapse rule.}\label{stab:credo}
\setlength{\tabcolsep}{3pt}\begin{tabular}{@{}llcrrrrrrrrrrrrr@{}}\toprule
Model, task & Seed & Excl. & Forced (\%) & Tokens & Median $g$ & $g>10$ (\%) & $g<-10$ (\%) & Acc. (\%) & Brier & AURC & $T$ & TL Brier & REL & RES & UNC \\ \midrule
Qwen3-0.6B, MMLU-Pro & 17 &  & 0.0 & 56 & $-$0.8 & 0.0 & 0.0 & 27.7 & 0.853 & 0.623 & 1.00 & 0.195 & 0.004 & 0.008 & 0.200 \\
Qwen3-0.6B, MMLU-Pro & 23 &  & 0.0 & 22 & $-$0.2 & 0.0 & 0.0 & 28.1 & 0.854 & 0.682 & 1.62 & 0.200 & 0.000 & 0.001 & 0.202 \\
Qwen3-0.6B, MMLU-Pro & 29 &  & 0.0 & 9 & $-$1.4 & 0.0 & 0.0 & 29.1 & 0.865 & 0.713 & 1.09 & 0.217 & 0.010 & 0.001 & 0.206 \\
\quad PRDO (Brier), mean &  &  &  &  &  &  &  & 35.5 & 0.757 & 0.444 &  & 0.186 & 0.001 & 0.042 & 0.229 \\
\quad RLCR+TS, mean &  &  &  &  &  &  &  & 29.9 & 0.848 & 0.635 &  & 0.207 & 0.006 & 0.008 & 0.210 \\
\midrule
Qwen3-0.6B, GSM8K-Verify & 17 &  & 2.1 & 139 & 19.0 & 100.0 & 0.0 & 81.1 & 0.303 & 0.126 & 13.58 & 0.151 & 0.003 & 0.005 & 0.153 \\
Qwen3-0.6B, GSM8K-Verify & 23 &  & 3.3 & 174 & 4.2 & 0.0 & 0.0 & 87.2 & 0.234 & 0.106 & 1.90 & 0.117 & 0.006 & 0.001 & 0.111 \\
Qwen3-0.6B, GSM8K-Verify & 29 &  & 17.6 & 218 & 3.2 & 0.0 & 0.0 & 87.1 & 0.222 & 0.080 & 1.72 & 0.111 & 0.001 & 0.002 & 0.113 \\
\quad PRDO (Brier), mean &  &  &  &  &  &  &  & 79.2 & 0.299 & 0.100 &  & 0.149 & 0.003 & 0.016 & 0.164 \\
\quad RLCR+TS, mean &  &  &  &  &  &  &  & 82.0 & 0.295 & 0.182 &  & 0.148 & 0.000 & 0.000 & 0.148 \\
\midrule
Qwen3-1.7B, MMLU-Pro & 17 &  & 2.4 & 117 & $-$0.1 & 0.0 & 0.0 & 42.7 & 0.714 & 0.345 & 1.41 & 0.190 & 0.004 & 0.058 & 0.245 \\
Qwen3-1.7B, MMLU-Pro & 23 & yes & 0.0 & 4 & $-$18.8 & 0.0 & 100.0 & 34.2 & 0.889 & 0.630 & 1.00 & 0.277 & 0.053 & 0.001 & 0.225 \\
Qwen3-1.7B, MMLU-Pro & 29 &  & 0.3 & 63 & 0.0 & 0.0 & 0.0 & 36.1 & 0.778 & 0.442 & 1.62 & 0.194 & 0.005 & 0.039 & 0.231 \\
\quad PRDO (Brier), mean &  &  &  &  &  &  &  & 44.4 & 0.670 & 0.317 &  & 0.181 & 0.002 & 0.067 & 0.247 \\
\quad RLCR+TS, mean &  &  &  &  &  &  &  & 41.8 & 0.767 & 0.478 &  & 0.229 & 0.008 & 0.022 & 0.243 \\
\midrule
Qwen3-1.7B, GSM8K-Verify & 17 &  & 2.2 & 146 & 25.6 & 100.0 & 0.0 & 92.9 & 0.133 & 0.079 & 9.28 & 0.067 & 0.000 & 0.000 & 0.066 \\
Qwen3-1.7B, GSM8K-Verify & 23 &  & 0.0 & 59 & 24.8 & 100.0 & 0.0 & 79.8 & 0.321 & 0.158 & 16.27 & 0.160 & 0.000 & 0.000 & 0.161 \\
Qwen3-1.7B, GSM8K-Verify & 29 &  & 0.6 & 137 & 23.4 & 100.0 & 0.0 & 94.6 & 0.102 & 0.054 & 7.30 & 0.051 & 0.000 & 0.000 & 0.051 \\
\quad PRDO (Brier), mean &  &  &  &  &  &  &  & 92.8 & 0.135 & 0.062 &  & 0.068 & 0.002 & 0.002 & 0.067 \\
\quad RLCR+TS, mean &  &  &  &  &  &  &  & 93.5 & 0.121 & 0.064 &  & 0.060 & 0.001 & 0.001 & 0.061 \\
\bottomrule\end{tabular}\end{table}

\end{landscape}
\clearpage
\begin{longtable}{@{}>{\raggedright\arraybackslash}p{4.4cm}>{\raggedright\arraybackslash}p{2.6cm}>{\raggedright\arraybackslash}p{2.6cm}>{\raggedright\arraybackslash}p{2.6cm}>{\raggedright\arraybackslash}p{2.6cm}@{}}
\caption{\textbf{Development results on the validation split (the earlier test split).} The analysis plan reports earlier results on this split as development results; no comparison of the paper uses them. Single-query readout after temperature scaling; Brier score / AURC, or accuracy (\%) where stated. \textbf{a}, Selection of GRPO+TS's KL setting (seed 17, the pre-specified rule: lowest mean Brier score over the four Qwen3 cells, within 0.002 the lowest mean AURC; a setting whose rationales collapse in any cell cannot be chosen). \textbf{b}, Selection of RLCR+TS's advantage rule by the same rule. \textbf{c}, Baselines run on this split only: RLOO and GRPO+TS at twice the budget (800 steps), mean $\pm$ s.d. over seeds 17, 23 and 29. \textbf{d}, The earlier development report (its own scoring; means over one to three runs), which compared the token-anchor weights 0.01 (the adopted setting) and 0.04, the earlier recipe whose sequence-level KL term was inert in stage 1, and further baselines on all four backbones and both tasks (RFT+KL+TS, rejection-sampling fine-tuning with a KL penalty; --, not run or excluded by the collapse rule).}\label{stab:dev}\\
\toprule Setting & Qwen3-0.6B, GSM8K-Verify & Qwen3-0.6B, MMLU-Pro & Qwen3-1.7B, GSM8K-Verify & Qwen3-1.7B, MMLU-Pro \\ \midrule \endfirsthead
\multicolumn{5}{@{}l}{\emph{Supplementary Table 6, continued}} \\ \midrule \endhead
\bottomrule \endfoot
\multicolumn{5}{@{}l}{\textbf{a} GRPO+TS setting (Brier / AURC)} \\
$\beta=0$ & 0.242 / 0.101 & collapsed & 0.151 / 0.076 & 0.762 / 0.498 \\
$\beta=0.01$ & 0.316 / 0.201 & collapsed & 0.129 / 0.057 & 0.765 / 0.497 \\
$\beta=0.04$ & 0.331 / 0.190 & 0.842 / 0.632 & 0.151 / 0.065 & 0.769 / 0.510 \\
$\beta=0.1$ & 0.333 / 0.206 & 0.844 / 0.632 & 0.149 / 0.062 & 0.760 / 0.490 \\
$\beta=0.04$, token anchor $\lambda=0.01$ (selected) & 0.315 / 0.185 & 0.837 / 0.625 & 0.171 / 0.081 & 0.754 / 0.478 \\
\midrule \multicolumn{5}{@{}l}{\textbf{b} RLCR+TS advantage rule (Brier / AURC)} \\
Standardized (GRPO) & 0.291 / 0.176 & 0.891 / 0.749 & 0.140 / 0.074 & 0.789 / 0.490 \\
Leave-one-out (selected) & 0.312 / 0.194 & 0.843 / 0.594 & 0.132 / 0.069 & 0.767 / 0.449 \\
\midrule \multicolumn{5}{@{}l}{\textbf{c} Further baselines (accuracy; Brier / AURC)} \\
RLOO, accuracy & 78.7$\pm$0.5 & 29.9$\pm$1.3 & 92.4$\pm$0.8 & 42.5$\pm$1.3 \\
RLOO, Brier / AURC & 0.341$\pm$0.010 / 0.226$\pm$0.016 & 0.837$\pm$0.011 / 0.611$\pm$0.040 & 0.141$\pm$0.009 / 0.065$\pm$0.004 & 0.751$\pm$0.011 / 0.481$\pm$0.012 \\
GRPO+TS, 800 steps, accuracy & 80.4$\pm$1.4 & 30.5$\pm$1.0 & 92.5$\pm$0.4 & 43.2$\pm$0.9 \\
GRPO+TS, 800 steps, Brier / AURC & 0.317$\pm$0.016 / 0.201$\pm$0.019 & 0.837$\pm$0.002 / 0.630$\pm$0.008 & 0.142$\pm$0.006 / 0.070$\pm$0.006 & 0.749$\pm$0.002 / 0.480$\pm$0.000 \\
\midrule \multicolumn{5}{@{}l}{\textbf{d} Earlier development report (accuracy, \%; Brier / AURC)} \\*
\emph{Arm} & \emph{Qwen3-0.6B, GSM8K-Verify} & \emph{Qwen3-1.7B, GSM8K-Verify} & \emph{SmolLM3-3B, GSM8K-Verify} & \emph{Phi-4-mini, GSM8K-Verify} \\*
Zero-shot & 65.1; 0.435 / 0.285 & 84.1; 0.211 / 0.076 & 90.2; 0.161 / 0.055 & 95.2; 0.095 / 0.036 \\
SFT+TS & 67.6; 0.417 / 0.228 & 70.4; 0.382 / 0.174 & 70.5; 0.386 / 0.183 & 73.3; 0.353 / 0.143 \\
GRPO+TS (earlier setting) & 82.7; 0.293 / 0.183 & 91.8; 0.150 / 0.067 & 90.6; 0.167 / 0.068 & 96.4; 0.067 / 0.034 \\
RFT+KL+TS & 75.1; 0.375 / 0.247 & 91.4; 0.160 / 0.073 & 87.7; 0.215 / 0.116 & 95.8; 0.078 / 0.034 \\
PRDO, earlier recipe & 81.4; 0.274 / 0.092 & 90.4; 0.164 / 0.047 & 89.5; 0.170 / 0.039 & 77.5; 0.311 / 0.105 \\
PRDO, anchor 0.04 & 77.4; 0.318 / 0.125 & 91.1; 0.167 / 0.073 & 91.4; 0.144 / 0.051 & 96.0; 0.074 / 0.019 \\
PRDO, anchor 0.01 & 78.1; 0.315 / 0.119 & 92.1; 0.147 / 0.069 & 92.5; 0.133 / 0.033 & 96.0; 0.070 / 0.021 \\
Two-stage, earlier recipe & 82.0; 0.265 / 0.086 & 92.2; 0.135 / 0.034 & 91.6; 0.139 / 0.031 & 73.3; 0.346 / 0.138 \\
Two-stage, anchor 0.04 & 78.8; 0.305 / 0.114 & 92.7; 0.135 / 0.061 & 92.4; 0.128 / 0.034 & 95.2; 0.082 / 0.021 \\
Two-stage, anchor 0.01 & 78.3; 0.304 / 0.108 & 93.1; 0.136 / 0.062 & 92.7; 0.127 / 0.029 & 96.2; 0.070 / 0.024 \\
\midrule \multicolumn{5}{@{}l}{\textbf{d} Earlier development report (accuracy, \%; Brier / AURC)} \\*
\emph{Arm} & \emph{Qwen3-0.6B, MMLU-Pro} & \emph{Qwen3-1.7B, MMLU-Pro} & \emph{SmolLM3-3B, MMLU-Pro} & \emph{Phi-4-mini, MMLU-Pro} \\*
Zero-shot & 24.6; 0.861 / 0.661 & 38.9; 0.786 / 0.516 & 43.2; 0.755 / 0.459 & 46.4; 0.719 / 0.403 \\
SFT+TS & 33.5; 0.778 / 0.482 & 42.3; 0.711 / 0.386 & 46.2; 0.675 / 0.321 & 51.1; 0.619 / 0.265 \\
GRPO+TS (earlier setting) & 31.0; 0.836 / 0.624 & 42.4; 0.753 / 0.481 & 43.2; 0.737 / 0.419 & 53.2; 0.639 / 0.323 \\
RFT+KL+TS & 30.4; 0.839 / 0.620 & 40.9; 0.761 / 0.494 & -- & 51.5; 0.656 / 0.343 \\
PRDO, earlier recipe & 33.3; 0.774 / 0.466 & 47.0; 0.667 / 0.312 & -- & -- \\
PRDO, anchor 0.04 & -- & 45.2; 0.685 / 0.336 & 48.3; 0.638 / 0.287 & 57.6; 0.579 / 0.228 \\
PRDO, anchor 0.01 & 33.4; 0.775 / 0.462 & 45.8; 0.681 / 0.333 & 50.6; 0.635 / 0.273 & 57.5; 0.563 / 0.210 \\
Two-stage, earlier recipe & 33.6; 0.775 / 0.471 & 47.7; 0.658 / 0.304 & -- & -- \\
Two-stage, anchor 0.04 & -- & 45.1; 0.686 / 0.341 & 52.1; 0.612 / 0.248 & 55.8; 0.568 / 0.220 \\
Two-stage, anchor 0.01 & 34.5; 0.768 / 0.456 & 47.2; 0.672 / 0.322 & 50.7; 0.619 / 0.259 & 57.3; 0.554 / 0.202 \\
\end{longtable}

\normalsize

\end{document}